\documentclass[10pt,letterpaper]{article}
\usepackage[preprint]{tmlr}
\usepackage{amsmath,amssymb,amsthm,mathtools}
\usepackage{graphicx,booktabs,microtype,float}
\usepackage{needspace}
\makeatletter
\renewcommand{\float@endH}{%
  \@endfloatbox
  \if@flstyle\setbox\@currbox\float@makebox\columnwidth\fi
  \Needspace{\dimexpr\ht\@currbox+\dp\@currbox+\intextsep\relax}%
  \vskip\intextsep
  \box\@currbox
  \vskip\intextsep\relax}
\makeatother
\usepackage[hidelinks]{hyperref}
\hypersetup{pdftitle={Correcting CondOT: Exact Finite-Step Sampling in Gaussian Flow Matching},pdfauthor={Ron Levy and Michael Elad}}
\newtheorem{theorem}{Theorem}
\newtheorem{proposition}[theorem]{Proposition}
\newtheorem{corollary}[theorem]{Corollary}
\newtheorem{lemma}[theorem]{Lemma}
\allowdisplaybreaks[1]
\DeclareMathOperator{\diag}{diag}

\title{Correcting CondOT: Exact Finite-Step Sampling\\
in Gaussian Flow Matching}
\author{\name Ron Levy \email ron.levy@campus.technion.ac.il \\
  \addr Department of Computer Science\\
  Technion, Israel
  \AND
  \name Michael Elad \email elad@cs.technion.ac.il \\
  \addr Department of Computer Science\\
  Technion, Israel}

\begin{document}
\maketitle

\begin{abstract}
Flow matching generates samples by gradually transforming noise into data. In practice, using a finite number of sampling steps introduces a numerical error that depends on the chosen schedule. We study this dependence for Gaussian targets and the explicit midpoint sampling method, using the exact flow field. We measure sampling error by the squared Wasserstein distance between the target distribution and the final distribution produced by the midpoint sampler. We show that the standard conditional optimal transport (CondOT) schedule cancels the leading midpoint error and improves the general convergence bound, even when the sampling steps are unequally spaced. On a uniform grid of \(S\) sampling steps, we fix the signal schedule at \(\alpha_t=t\) and prove the existence of scalar noise schedules \(\beta_t\) that approach the CondOT noise schedule \(1-t\) at rate \(1/S\) and yield exact Gaussian sampling for every sufficiently large \(S\). Controlled Gaussian experiments illustrate the convergence rates and exact calibration.\footnote{The authors used LLMs for assistance with LaTeX formatting, manuscript presentation, and code generation. All scientific ideas, claims, derivations, results, and final content were reviewed and approved by the authors.}
\end{abstract}

\section{Introduction}
\label{m:introduction}
In flow matching \citep{lipman2023}, samples are generated by following a learned transformation from noise to a target distribution. The schedule controls the contributions of noise and data along this transformation, while a numerical sampler approximates it using a finite number of steps. Although different schedules can reach the same target when the dynamics are solved exactly, their numerical approximations need not have the same accuracy. This raises a natural question: how does the sampling error depend on the schedule, and how should the schedule be chosen for a given sampling budget?

To study this dependence, we consider Gaussian targets, whose closed-form flow fields allow us to isolate the numerical error of the explicit midpoint sampler from neural-network approximation error. Although a known Gaussian can be sampled directly, this setting allows us to examine how schedule choice affects numerical accuracy and how changing a shared scalar schedule can compensate for the sampler’s error.

We first derive the distribution produced by midpoint sampling under general schedules and arbitrary prescribed sampling times. We then show that increasing the data contribution linearly is the unique signal schedule that preserves the terminal mean for every target mean, prescribed time grid, and noise schedule satisfying our assumptions. Fixing this choice removes the mean error and gives a simple objective based on the remaining covariance mismatch between the target distribution and the final distribution produced by midpoint sampling. 

Building on this objective, we analyze how the error decreases as the sampling steps become finer. For the standard conditional optimal transport (CondOT) schedule, which increases the signal contribution and decreases
the noise contribution linearly, the leading error term cancels, producing a faster convergence rate. This effect, called \emph{superconvergence}, holds even when the sampling steps are unequally spaced.

The cancellation does not generally make CondOT exact at a finite
sampling budget. We therefore ask whether the remaining error can
be eliminated by a small change to the default schedule. On equally spaced grids, we keep the signal schedule fixed and add a polynomial correction to CondOT's
noise schedule. The number of coefficients equals the
number of distinct eigenvalues of the target covariance. We prove that, for every sufficiently large sampling budget $S$,
a correction to the default noise schedule of order $1/S$ produces a schedule
that makes the midpoint sampler exactly reproduce the Gaussian target.

Finally, Gaussian experiments compare CondOT with other noise
schedules, confirm its faster convergence rate, and demonstrate
that the calibrated correction reduces the terminal error to
numerical precision at finite sampling budgets.

\section{Problem Setup and Exact Sampling Error}
\label{m:setup}

We consider Gaussian flow matching and study how the interpolation schedules affect the error of midpoint sampling.

\paragraph{Gaussian probability path.}
Let $z\sim p_{\mathrm{data}}=\mathcal N(\mu,\Sigma)$, with
$\Sigma\succ0$, and independently draw
$\epsilon\sim p_{\mathrm{init}}=\mathcal N(0,I)$.
We connect noise to data through
\begin{equation}
 X_t=\alpha_t z+\beta_t\epsilon,
 \qquad 0\leq t\leq1.
 \label{m:path}
\end{equation}
The signal schedule $\alpha_t$ controls the contribution of the
target sample, while the noise schedule $\beta_t$ controls the
remaining noise. We impose
\[
 \alpha_0=0,\qquad \alpha_1=1,\qquad
 \beta_0=1,\qquad \beta_1=0,
\]
so that $X_0=\epsilon$ and $X_1=z$.

Both schedules must also belong to $C^1([0,1])$ and satisfy
\[
 \alpha_t^2+\beta_t^2>0
 \qquad\text{for every }t\in[0,1].
\]

Additional regularity is stated where needed.

Because the noise and target sample are independent Gaussians,
the distribution at every intermediate time is
\begin{equation}
 p_t=\mathcal N\!\left(
 \alpha_t\mu,\beta_t^2I+\alpha_t^2\Sigma
 \right).
 \label{m:marginalpath}
\end{equation}
Thus the schedules determine both the mean and the covariance
along the path.

\paragraph{The exact flow.}
Flow matching generates samples by following a velocity field
whose flow has the prescribed distributions $p_t$.
For the Gaussian path above, the exact target field
$u_t^{\mathrm{target}}(x)=\mathbb E[\dot X_t\mid X_t=x]$
has the explicit form
\begin{equation}
\begin{aligned}
 u_t^{\mathrm{target}}(x)
 &=\dot\alpha_t\mu +
 \left(\beta_t\dot\beta_t I+
       \alpha_t\dot\alpha_t\Sigma\right)
 \left(\beta_t^2I+\alpha_t^2\Sigma\right)^{-1}
 (x-\alpha_t\mu).
\end{aligned}
 \label{m:exactfield}
\end{equation}
Starting from $Y_0=\epsilon$, the corresponding dynamics
\begin{equation}
 \frac{dY_t}{dt}=u_t^{\mathrm{target}}(Y_t)
 \label{m:exactflow}
\end{equation}
satisfy $Y_t\sim p_t$ and therefore reach
$p_{\mathrm{data}}$ at $t=1$.
Appendix~\ref{sec:gaussian-fm} derives this field.

\paragraph{Midpoint sampling.}
Fix prescribed sampling times $0=t_0<\cdots<t_S=1$, and write
\[
 h_n=t_{n+1}-t_n,
 \qquad
 t_{n+\frac12}=t_n+\frac{h_n}{2},
 \qquad
 h_{\max}=\max_{0\leq n<S}h_n.
\]
Starting from
$Y_0^{\mathrm M}=\epsilon$, the explicit midpoint method
first predicts a state halfway through the current step,
then uses the field at that predicted state to complete it:
\begin{equation}
\begin{aligned}
 Y_{n+\frac12}^{\mathrm M}
 &=Y_n^{\mathrm M}
   +\frac{h_n}{2}
    u_{t_n}^{\mathrm{target}}(Y_n^{\mathrm M}),\\
 Y_{n+1}^{\mathrm M}
 &=Y_n^{\mathrm M}
   +h_n u_{t_{n+\frac12}}^{\mathrm{target}}
         (Y_{n+\frac12}^{\mathrm M}).
\end{aligned}
 \label{m:midpoint}
\end{equation}
Each step requires two field evaluations, so
$N_{\mathrm{FE}}=2S$.
The uniform grid is the special case $t_n=nh$, with $h=S^{-1}$.

Let $p_S^{\mathrm M}$ denote the distribution of
$Y_S^{\mathrm M}$. We measure its discrepancy from the target by
\begin{equation}
 J_S=W_2^2(p_S^{\mathrm M},p_{\mathrm{data}}).
 \label{m:errorcriterion}
\end{equation}
Our goal is to choose schedules that reduce this error for
a prescribed sampling grid. All reported convergence rates
shown later concern this squared distance.

\subsection{A signal schedule that preserves the mean}

The terminal error can arise from an incorrect mean or
covariance. We first identify the signal schedule that
eliminates the mean error, then obtain an exact expression
for the remaining covariance error.

\begin{theorem}
\label{m:mean}
Fix $\Sigma\succ0$ and a signal schedule
$\alpha\in C^1([0,1])\cap C^2((0,1))$
with $\alpha_0=0$ and $\alpha_1=1$, chosen independently of
$S$, the sampling grid, and $\beta$.
The midpoint terminal mean satisfies
\[
 \mathbb E[Y_S^{\mathrm M}]=\mu
\]
for every $S\geq1$, every prescribed grid, every target mean $\mu$,
and every noise schedule $\beta$ satisfying the assumptions above
with this fixed $\alpha$, if and only if $\alpha_t=t$.

For this choice, midpoint sampling also preserves the exact
mean at every grid point and predicted midpoint:
\[
 \mathbb E[Y_n^{\mathrm M}]=t_n\mu,
 \qquad
 \mathbb E[Y_{n+\frac12}^{\mathrm M}]
 =t_{n+\frac12}\mu.
\]
\end{theorem}

The proof is given in
Appendix~\ref{subsec:original-mean-proof}. From now on, we fix $\alpha_t=t$, which guarantees
the correct mean at every sampling budget.

\subsection{The exact finite-step objective}

We now express the output covariance and sampling error
in terms of the noise schedule. Write
\[
 \Sigma^{-1}
 =U\diag(\lambda_1,\ldots,\lambda_d)U^\top,
 \qquad \lambda_k>0,
\]
where $U$ is orthogonal. The target variance in direction $k$
is $1/\lambda_k$.

The following result gives the exact output distribution and
its squared Wasserstein distance from the target.

\begin{proposition}
\label{m:exactobjective}
For $\alpha_t=t$, every noise schedule $\beta\in C^1([0,1])$
with $\beta_0=1$ and $\beta_1=0$, and every prescribed sampling
grid, the midpoint output is
\begin{equation}
\begin{aligned}
 p_S^{\mathrm M}&=\mathcal N(\mu,\Sigma_{\mathrm M}),\\
 \Sigma_{\mathrm M}
 &=U\diag\!\left(
 (D_{1,1}^{\mathrm M})^2,\ldots,
 (D_{1,d}^{\mathrm M})^2
 \right)U^\top,
\end{aligned}
 \label{m:outputlaw}
\end{equation}
where
\begin{equation}
 D_{1,k}^{\mathrm M}
 =\prod_{n=0}^{S-1}
 \left[
 1+h_nA_k(t_{n+\frac12})
 +\frac{h_n^2}{2}A_k(t_{n+\frac12})A_k(t_n)
 \right],
 \label{m:products}
\end{equation}
with
\begin{equation}
 v_k(t)=\beta_t^2+\frac{t^2}{\lambda_k},
 \qquad
 A_k(t)=\frac{\dot v_k(t)}{2v_k(t)}.
 \label{m:scalarfield}
\end{equation}
Consequently, the squared Wasserstein error is
\begin{equation}
 J_S
 =\sum_{k=1}^d
 \left(
 |D_{1,k}^{\mathrm M}|-\frac1{\sqrt{\lambda_k}}
 \right)^2.
 \label{m:objective}
\end{equation}
\end{proposition}

Here, $v_k(t)$ is the variance of the exact probability path
in direction $k$, while $|D_{1,k}^{\mathrm M}|$ is the
standard deviation produced by the sampler in that direction. In the asymptotic analysis below, $D_{1,k}^{\mathrm M}>0$
for sufficiently fine grids, so the absolute values in
\eqref{m:objective} can be omitted; see
Appendices~\ref{subsec:midpoint-positivity}
and~\ref{app:calibration-uniform-expansion}.

Appendix~\ref{subsec:original-midpoint-output} derives the
output distribution and the product in \eqref{m:products}.
Appendix~\ref{sec:objectives} derives the Wasserstein expression.
The objective is exact for every sampling budget and provides
the basis for the convergence analysis and schedule optimization
below.
\section{CondOT Superconvergence}
\label{m:analysis}

We now study how the sampling error decreases as the time
steps become smaller. Throughout this section, we keep the
target covariance $\Sigma$ and the noise schedule fixed,
and refine the sampling grid so that $h_{\max}\to0$.

For a fixed noise schedule satisfying the assumptions of
Section~\ref{m:setup}, with $v_k\in C^4([0,1])$ and $v_k$
bounded away from zero
in every direction, the general midpoint error bound is
\[
 J_S=O(h_{\max}^4).
\]
Appendix~\ref{subsec:original-local-defect} derives this
bound and the corresponding error expansion.
Our main result shows that CondOT, which uses
$\beta_t=1-t$, cancels the leading error term and achieves
a stronger bound.

\begin{theorem}
\label{m:condot}
For $\alpha_t=t$ and $\beta_t=1-t$, along every sequence
of prescribed grids with $h_{\max}\to0$,
\begin{equation}
 D_{1,k}^{\mathrm M}-\frac1{\sqrt{\lambda_k}}
 =O(h_{\max}^3),
 \qquad
 J_S=O(h_{\max}^6).
 \label{m:condotgeneral}
\end{equation}
\end{theorem}

The proof is given in
Appendix~\ref{subsec:original-condot-general}.
Thus CondOT improves the general squared Wasserstein
error bound from $O(h_{\max}^4)$ to $O(h_{\max}^6)$.
We refer to this improvement as \emph{superconvergence}.
The result does not require equally spaced sampling times.

\paragraph{The remaining error on a uniform grid.}
For equally spaced steps, we can also calculate the
leading coefficient of the remaining error.
This gives a more precise description of CondOT's accuracy
and provides the basis for exact schedule calibration
in the next section.

\begin{corollary}
\label{m:condotuniformexpansion}
For CondOT on the uniform grid, with $h=S^{-1}$, the following expansions hold as $S\to\infty$,
\begin{align}
 D_{1,k}^{\mathrm M}-\frac1{\sqrt{\lambda_k}}
 &=\frac{d_k}{\sqrt{\lambda_k}}h^3+O(h^4),
 \label{m:condotscale}\\
 J_S
 &=h^6\sum_{k=1}^d\frac{d_k^2}{\lambda_k}+O(h^7),
 \label{m:condotuniform}
\end{align}
where
\begin{equation}
 d_k
 =-\frac1{24}
  -\frac{\pi(1+\lambda_k)^3}{256\lambda_k^{3/2}}
 <0.
 \label{m:dk}
\end{equation}
\end{corollary}

Appendix~\ref{subsec:original-condot-uniform} derives
these expansions and evaluates $d_k$.
Since the coefficient in \eqref{m:condotuniform} is strictly
positive, CondOT has squared Wasserstein error
$\Theta(S^{-6})$ on the uniform grid.

The sign of $d_k$ also describes the remaining discrepancy:
for sufficiently large $S$, the generated standard deviation
is smaller than the target standard deviation in every
eigendirection. In the next section, we ask whether a small
change to the noise schedule can eliminate this remaining error.

\newpage
\section{Exact Finite-Budget Calibration Near CondOT}
\label{m:calibration}

On a uniform grid, we keep $\alpha_t=t$ and add a polynomial
correction to CondOT's noise schedule, with one coefficient per
distinct precision eigenvalue. We choose these coefficients jointly
to eliminate the remaining error in every eigendirection using
a single shared noise schedule.
Let $r$ be the number of distinct precision eigenvalues of the fixed
covariance. After reordering, let $\lambda_1,\ldots,\lambda_r$ contain
one representative of each distinct value, and set
\begin{equation}
 \beta_{\boldsymbol\theta}(t)
 =(1-t)\left(1+\sum_{j=1}^r\theta_jt^j\right).
 \label{m:calibration-family}
\end{equation}
This family reduces to CondOT at $\boldsymbol\theta=0$.
Every coefficient vector gives a smooth schedule with the required
endpoints.

\begin{theorem}
\label{m:exact-calibration}
Fix $\Sigma\succ0$. There is a unique vector
$\boldsymbol K=(K_1,\ldots,K_r)^\top$ solving
\begin{equation}
 \sum_{j=1}^r K_j\int_0^1
 \frac{\big[2t^j(1-t)^2\big]'''}
 {(1-t)^2+t^2/\lambda_k}\,dt
 =48d_k,\qquad k=1,\ldots,r,
 \label{m:calibration-linear-system}
\end{equation}
where $d_k$ is the CondOT coefficient in \eqref{m:dk}.
The first-order correction satisfies
\begin{equation}
 J_S\!\left(\beta_{\boldsymbol K/S}\right)=O(S^{-8})
 \qquad\text{as }S\to\infty.
 \label{m:calibration-leading}
\end{equation}
Moreover, for every sufficiently large integer $S$, there is a locally
unique parameter vector
$\boldsymbol\theta_S=\boldsymbol K/S+O(S^{-2})$
within this polynomial family such that
\begin{equation}
 J_S\!\left(\beta_{\boldsymbol\theta_S}\right)=0.
 \label{m:calibration-exact}
\end{equation}
Equivalently, the corresponding midpoint output satisfies
$Y_S^{\mathrm M}=\mu+\Sigma^{1/2}\epsilon$.
\end{theorem}

The two conclusions require different amounts of calibration.
Solving the linear system~\eqref{m:calibration-linear-system} once gives
$\boldsymbol K$ and the explicit schedule $\beta_{\boldsymbol K/S}$.
This cancels the leading error in every eigendirection, giving
\eqref{m:calibration-leading}. To obtain zero loss at a particular
budget, one instead solves the exact equations
\begin{equation}
 D_{1,k}^{\mathrm M}(\boldsymbol\theta_S)
 =\frac1{\sqrt{\lambda_k}},\qquad k=1,\ldots,r,
 \label{m:calibration-equations}
\end{equation}
where $D_{1,k}^{\mathrm M}$ is the midpoint product
in \eqref{m:products}, evaluated with \eqref{m:calibration-family}.
Solving these $r$ nonlinear equations determines the
$O(S^{-2})$ adjustment to $\boldsymbol K/S$. Repeated eigenvalues
require no additional equations.
Such a solution exists for every sufficiently large $S$,
with the calibrated coefficients and the budget threshold depending
on the target spectrum.
Appendix~\ref{app:exact-calibration} proves this result by constructing
an iteration that converges to the calibrated coefficients from
$\boldsymbol K/S$ for sufficiently large $S$.

\section{Numerical Experiments}
\label{m:experiments}

We use four centered Gaussian targets to examine two features of the
calibration problem: the number $r$ of distinct eigenvalues, which determines
the number of correction coefficients, and the spectral spread. All targets
have dimension $d=6$ and total variance $\operatorname{tr}(\Sigma)=6$.
For a positive vector $\boldsymbol w$, set
\begin{equation}
 \Sigma(\boldsymbol w)
 =\frac{6}{\sum_{i=1}^6w_i}\diag(w_1,\ldots,w_6).
 \label{m:spectrum}
\end{equation}
Table~\ref{tab:targets} specifies the four targets. The normalization preserves
the condition number $\kappa=\lambda_{\max}(\Sigma)/\lambda_{\min}(\Sigma)$.
The two-eigenvalue cases use each variance three times, while the
six-eigenvalue cases use geometric spacing. This gives a structured comparison
at fixed dimension and total variance.

\begin{table}[H]
 \centering
 \caption{Four Gaussian targets. The vectors $\boldsymbol w$ are normalized
 by~\eqref{m:spectrum}.}
 \label{tab:targets}
 \begin{tabular}{cccl}
 \toprule
 Target & $r$ & $\kappa$ & Unnormalized covariance eigenvalues $\boldsymbol w$\\
 \midrule
 A & 2 & $4$ & $(1,1,1,4,4,4)$\\
 B & 6 & $4$ & $(4^{j/5})_{j=0}^{5}$\\
 C & 2 & $10^4$ & $(1,1,1,10^4,10^4,10^4)$\\
 D & 6 & $10^4$ & $(10^{4j/5})_{j=0}^{5}$\\
 \bottomrule
 \end{tabular}
\end{table}

We fix $\alpha_t=t$, use the exact Gaussian field, and sample on uniform
grids. A budget of $S$ midpoint steps uses $2S$ field evaluations.
High-precision losses are evaluated from~\eqref{m:objective} using
90-digit decimal arithmetic. Sections~\ref{m:fixed-experiment}
and~\ref{m:calibration-experiment} use Target D as a common example,
combining six matching constraints with a broad spectrum.
Section~\ref{m:crossing-experiment} compares the calibrated schedules for
all four targets. Additional convergence results and the double-precision
check are in Appendices~\ref{app:four-target-results}
and~\ref{app:precision-check}.

\subsection{CondOT superconvergence}
\label{m:fixed-experiment}

We compare CondOT, $\beta_t=1-t$, with four fixed noise schedules:
\[
 (1-t)^2,\qquad 1-t^2,\qquad
 \cos(\pi t/2),\qquad (1-t)(1+0.1t).
\]
These alternatives satisfy the
assumptions of Section~\ref{m:analysis} and have
$J_S=\Theta(S^{-4})$ on uniform grids
(Appendix~\ref{app:squared-noise-baseline}), while CondOT has
$J_S=\Theta(S^{-6})$.
Figure~\ref{fig:accuracy} compares their losses for Target D at
$S=64,128,\ldots,65536$.
The curves approach the predicted rates. 

\begin{figure}[H]
 \centering
 \includegraphics[width=0.94\linewidth]{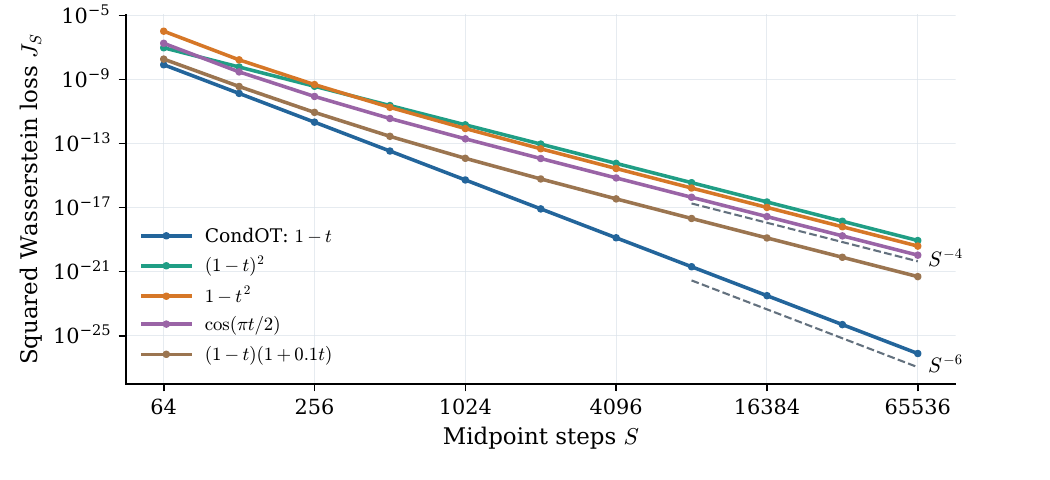}
 \caption{Fixed schedules for Target D. The four alternatives approach
 $S^{-4}$ squared Wasserstein decay, while CondOT achieves $S^{-6}$.
 Dashed lines indicate the asymptotic rates.}
 \label{fig:accuracy}
\end{figure}

\subsection{From first-order correction to exact calibration}
\label{m:calibration-experiment}

For the same target, we solve~\eqref{m:calibration-linear-system} once
for the six coefficients $\boldsymbol K$.
At each budget, we use the explicit schedule
$\beta_{\boldsymbol K/S}$ from~\eqref{m:calibration-family}.
Figure~\ref{fig:leading-correction} compares this schedule with CondOT
on the same grids as Figure~\ref{fig:accuracy}.
The corrected loss approaches eighth-order decay, consistent with
Theorem~\ref{m:exact-calibration}. The improvement is asymptotic.

\begin{figure}[H]
 \centering
 \includegraphics[width=0.94\linewidth]{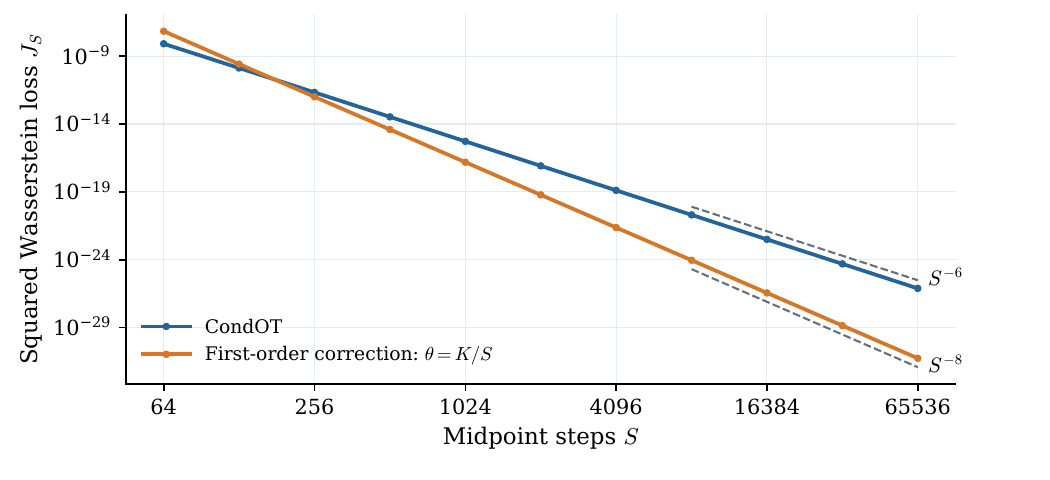}
 \caption{First-order correction for Target D. The explicit choice
 $\boldsymbol\theta=\boldsymbol K/S$ improves CondOT's $S^{-6}$ decay
 to $S^{-8}$. Dashed lines indicate these rates.}
 \label{fig:leading-correction}
\end{figure}

We then solve the exact matching equations~\eqref{m:calibration-equations}
using the coefficient iteration constructed in the proof
(Appendix~\ref{app:calibration-uniform-expansion}), starting from
$\boldsymbol K/S$.
Table~\ref{tab:calibration} compares the leading correction with the
numerically calibrated schedules at four budgets. The iteration reduces
the terminal loss below $10^{-170}$ in each case.
Appendix~\ref{app:calibration-numerics} gives the numerical procedure,
iteration counts, and results for all four targets, including failures
to reach the matching tolerance for Targets C and D at $S=64,128$.

\begin{table}[H]
 \centering
 \caption{Calibration for Target D using the proof's iteration.
 The last column reports numerical
 squared Wasserstein loss in 90-digit arithmetic.}
 \label{tab:calibration}
 \begin{tabular}{rrr}
 \toprule
 $S$ & $S^8J_S(\beta_{\boldsymbol K/S})$
 & $J_S(\beta_{\boldsymbol\theta_S})$ (numerical)\\
 \midrule
 $256$ & $1.8642\times10^{7}$ & $<10^{-170}$\\
$512$ & $1.8610\times10^{7}$ & $<10^{-170}$\\
$1024$ & $1.8602\times10^{7}$ & $<10^{-170}$\\
$2048$ & $1.8601\times10^{7}$ & $<10^{-170}$\\
 \bottomrule
 \end{tabular}
\end{table}

\subsection{How the target spectrum changes the calibrated schedule}
\label{m:crossing-experiment}

We now compare all four targets at the common budget $S=256$.
Figure~\ref{fig:four-crossings} shows the difference between each calibrated
noise schedule and CondOT. All panels use the same time and vertical scales.
The endpoints remain fixed, while the shape and size of the interior
correction depend on the target spectrum.
For these four targets, the calibrated corrections are larger for the broader spectra at
this budget: the maximum deviations are approximately
$2.42\times10^{-4}$, $8.15\times10^{-4}$, $3.15\times10^{-2}$,
and $4.07\times10^{-2}$ for A--D, respectively.
The two-eigenvalue corrections remain below CondOT in the interior,
whereas the six-eigenvalue examples cross it. These are features of the
selected targets, rather than universal consequences of $r$ or $\kappa$.
The coefficients are calibrated using the same procedure in every case.
Appendix~\ref{app:four-crossing-numerics} records the coefficients and
numerical deviations from CondOT.

\begin{figure}[H]
 \centering
 \includegraphics[width=\linewidth]{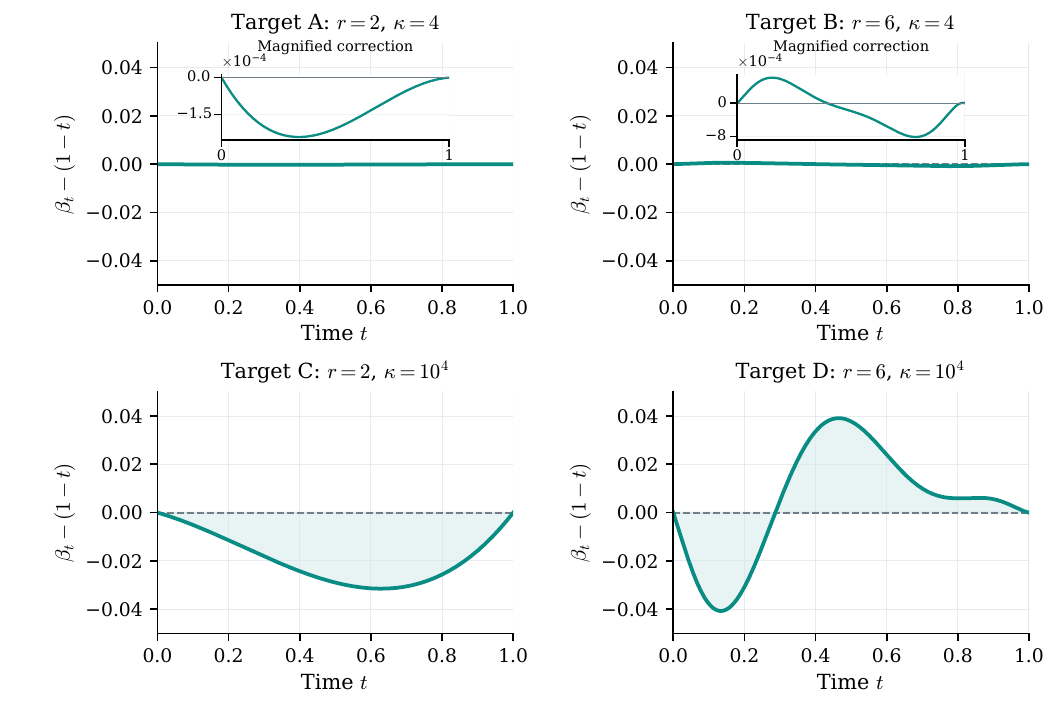}
 \caption{Calibrated schedule corrections at $S=256$.
 Top: modest spectral spread ($\kappa=4$). Bottom: broad spectral spread
 ($\kappa=10^4$). Left: two distinct eigenvalues. Right: six distinct
 eigenvalues. The horizontal line is CondOT, and all panels share the same
 vertical scale. The smaller corrections for Targets A and B are magnified.}
 \label{fig:four-crossings}
\end{figure}

\section{Related Work and Scope}
\label{m:related}

Flow matching learns generative vector fields through conditional regression \citep{lipman2023,guide2024,tong2024}, with close connections to stochastic interpolants \citep{albergo2025}. Generating samples requires numerically integrating the resulting ODE, introducing discretization error even when the exact vector field is available.

Gaussian models provide a tractable setting for analyzing the sampling
error. \citet{hurault2025} quantify Wasserstein errors in Gaussian
diffusion models, while \citet{benita2025spectral} derive exact
Gaussian output laws and optimize noise schedules using terminal
Wasserstein distance and KL divergence. \citet{chen2025} approach
schedule design through the vector field's Lipschitz regularity.
More directly related to discretization,
\citet{hurault2026geometry} derive Euler--Maruyama error expansions
and construct schedules that cancel the leading error in one
covariance eigendirection. 

Exact finite-step Gaussian sampling has also been studied.
For centered Gaussians with commuting covariances,
\citet{claustre2026gaussian} allow direction-dependent schedules
and show that Euler integrates the Gaussian Wasserstein geodesic
exactly on any time grid. Our setting uses midpoint sampling with
$\alpha_t=t$ and a single scalar noise schedule shared across all
covariance directions, and establishes exact terminal matching using corrections that vanish toward CondOT as the
sampling budget grows.

Mean exactness connects our analysis to \citet{bbdmschedules},
who study schedule design for Brownian Bridge Diffusion Models
under a Gaussian-mixture prior. Their surrogate reconstruction law
preserves each posterior component's mean, simplifying their
Wasserstein upper-bound objective to a weighted sum of squared
differences between generated and target standard deviations.
Fixing $\alpha_t=t$ gives the analogous simplification in our
Gaussian setting: the exact squared Wasserstein
objective~\eqref{m:objective} depends only on the covariance mismatch.

\section{Conclusion}

We analyzed how schedule choice affects midpoint sampling for Gaussian
flow matching. The linear signal schedule preserves the mean, and the
CondOT noise schedule cancels the leading numerical error, yielding
superconvergence even on nonuniform grids. On uniform grids, the remaining
error can be removed by a vanishing polynomial correction with one
coefficient per distinct precision eigenvalue. Its leading approximation improves the convergence rate beyond CondOT, while exact calibration gives zero terminal loss for every sufficiently large budget. This establishes exact finite-budget sampling within a family of scalar schedules approaching CondOT.

The analysis assumes a specific solver and a Gaussian
target with positive definite covariance.
Future work includes extending the analysis to Gaussian mixtures,
studying other numerical solvers, and examining how network approximation
error affects the benefits of calibration.

\label{m:mainend}
\section*{Acknowledgements}
This research was partially supported by the Israel Science Foundation (ISF) under Grants 951/24 and 409/24, and by the Council for Higher Education--Planning and Budgeting Committee.

\bibliography{references}
\bibliographystyle{tmlr}

\clearpage
\appendix
\numberwithin{equation}{section}
\numberwithin{theorem}{section}

\section{Problem Setup and Exact Sampling Error}
\label{app:setup-and-error}

This appendix supplies the derivations for Section~\ref{m:setup}.
We first obtain the Gaussian field, then establish universal mean exactness,
and finally derive the exact finite-step objective. 

\subsection{Gaussian probability path and exact field}
\label{sec:gaussian-fm}

Let the target distribution be
\begin{equation}
\label{eq:gaussian-data}
p_{\mathrm{data}}(z)=\mathcal N(z;\mu,\Sigma),
\qquad
\Sigma\succ0,
\end{equation}
and let the base distribution be
\[
p_{\mathrm{init}}=\mathcal N(0,I).
\]

\subsubsection{Gaussian probability path}

Draw
\[
z\sim p_{\mathrm{data}},
\qquad
\epsilon\sim\mathcal N(0,I),
\qquad
\epsilon\perp z.
\]
Let $\alpha,\beta\in C^1([0,1])$ satisfy the endpoint conditions
\begin{equation}
\label{eq:schedule-endpoints}
\alpha_0=0,
\qquad
\beta_0=1,
\qquad
\alpha_1=1,
\qquad
\beta_1=0,
\end{equation}
and the nondegeneracy condition
\[
\alpha_t^2+\beta_t^2>0
\qquad\text{for every }t\in[0,1].
\]

Define the interpolation
\begin{equation}
\label{eq:gaussian-path}
X_t=\alpha_tz+\beta_t\epsilon,
\qquad
0\leq t\leq1.
\end{equation}
Its data-conditional law is
\begin{equation}
\label{eq:conditional-path}
p_t(\cdot\mid z)=\mathcal N(\alpha_tz,\beta_t^2I),
\end{equation}
where, whenever $\beta_t=0$, the right-hand side is understood
as the Dirac measure $\delta_{\alpha_tz}$. In particular,
$p_1(\cdot\mid z)=\delta_z$.

Since $z$ and $\epsilon$ are independent Gaussian random vectors,
the marginal law $p_t$ of $X_t$ is
\begin{equation}
\label{eq:gaussian-marginal-path}
p_t=\mathcal N(\alpha_t\mu,\beta_t^2I+\alpha_t^2\Sigma).
\end{equation}
The endpoint conditions give
\[
p_0=p_{\mathrm{init}},
\qquad
p_1=p_{\mathrm{data}}.
\]

\subsubsection{Marginal vector field and exact flow}

We derive the marginal vector field directly from the interpolation
\eqref{eq:gaussian-path}. Define its marginal covariance by
\[
C_t\coloneqq\beta_t^2I+\alpha_t^2\Sigma.
\]
Since $\Sigma\succ0$ and $\alpha_t^2+\beta_t^2>0$,
$C_t$ is positive definite for every $t\in[0,1]$.

Differentiating the interpolation gives
\[
\dot X_t=\dot\alpha_tz+\dot\beta_t\epsilon.
\]
Define the marginal vector field by
\begin{equation}
\label{eq:marginal-field-definition}
u_t^{\mathrm{target}}(x)
\coloneqq\mathbb E[\dot X_t\mid X_t=x].
\end{equation}
Both $X_t$ and $\dot X_t$ are linear transformations of the
independent Gaussian random vectors $z$ and $\epsilon$, so they
are jointly Gaussian.

\paragraph{Gaussian regression.}
For jointly Gaussian random vectors $(A,B)$ with invertible
$\operatorname{Cov}(B)$,
\[
\mathbb E[A\mid B=b]
=\mathbb E[A]
+\operatorname{Cov}(A,B)\operatorname{Cov}(B)^{-1}
\bigl(b-\mathbb E[B]\bigr).
\]
Here,
\[
\mathbb E[X_t]=\alpha_t\mu,
\qquad
\mathbb E[\dot X_t]=\dot\alpha_t\mu,
\]
and independence gives
\[
\operatorname{Cov}(X_t)=C_t,
\qquad
\operatorname{Cov}(\dot X_t,X_t)
=\beta_t\dot\beta_t I+\alpha_t\dot\alpha_t\Sigma.
\]
Applying the regression identity therefore yields
\begin{equation}
\label{eq:exact-marginal-field}
u_t^{\mathrm{target}}(x)
=\dot\alpha_t\mu
+\left(\beta_t\dot\beta_t I+\alpha_t\dot\alpha_t\Sigma\right)
\left(\beta_t^2I+\alpha_t^2\Sigma\right)^{-1}
\left(x-\alpha_t\mu\right).
\end{equation}

\paragraph{Exact flow.}
Consider the target ODE
\begin{equation}
\label{eq:population-ode}
Y_0=\epsilon\sim\mathcal N(0,I),
\qquad
\frac{dY_t}{dt}=u_t^{\mathrm{target}}(Y_t).
\end{equation}
The field is continuous in $t$ and globally Lipschitz in $x$,
uniformly on $[0,1]$, because its affine coefficients are continuous
and $C_t^{-1}$ is bounded on this interval. Hence the ODE has a
unique solution.

We verify directly that this solution is
\[
Y_t=\alpha_t\mu+C_t^{1/2}\epsilon,
\]
where $C_t^{1/2}$ denotes the positive definite square root.
Indeed, $C_t$ is diagonal in the fixed eigenbasis of $\Sigma$,
so differentiation of its eigenvalues gives
\[
\frac{d}{dt}C_t^{1/2}
=\left(\beta_t\dot\beta_t I+\alpha_t\dot\alpha_t\Sigma\right)
C_t^{-1/2}.
\]
Consequently,
\begin{align*}
\frac{dY_t}{dt}
&=\dot\alpha_t\mu
+\left(\beta_t\dot\beta_t I+\alpha_t\dot\alpha_t\Sigma\right)
C_t^{-1/2}\epsilon
=u_t^{\mathrm{target}}(Y_t).
\end{align*}
Moreover, $C_0=I$ and $\alpha_0=0$, so the initial condition is
satisfied. It follows that
\[
Y_t\sim\mathcal N(\alpha_t\mu,C_t)=p_t
\qquad\text{for every }t\in[0,1],
\]
and in particular $Y_1\sim p_{\mathrm{data}}$.

\subsubsection{Spectral form of the marginal ODE}

For one Gaussian, the dynamics can be exposed directly in the
eigenbasis of the target covariance.  Diagonalize the target precision as

\begin{equation}
\label{eq:precision-basis}
\Sigma^{-1}
=
U\diag(\lambda_1,\ldots,\lambda_d)U^\top,
\qquad
\lambda_k>0.
\end{equation}
Define
\[
Y_t^U=U^\top Y_t,
\qquad
\epsilon^U=U^\top\epsilon,
\qquad
\mu^U=U^\top\mu.
\]
Then $\epsilon^U\sim\mathcal N(0,I)$ and the ODE decouples by
eigendirection:
\begin{equation}
\label{eq:scalar-gaussian-ode}
\frac{dY_{t,k}^U}{dt}
=
\dot\alpha_t\mu_k^U
+
\frac{
\beta_t\dot\beta_t+\alpha_t\dot\alpha_t/\lambda_k
}{
\beta_t^2+\alpha_t^2/\lambda_k
}
\left(Y_{t,k}^U-\alpha_t\mu_k^U\right).
\end{equation}
This scalar equation is the starting point
for the explicit-midpoint analysis below.

\subsection{A signal schedule that preserves the mean}
\label{sec:midpoint-analysis}

This subsection proves Theorem~\ref{m:mean}. We first record the
midpoint recursion and its general output law, which the
mean-exactness proof uses.
For the finite-step analysis, consider an
arbitrary prescribed time grid
\begin{equation}
\label{eq:general-time-grid}
0=t_0<t_1<\cdots<t_S=1,
\qquad
h_n=t_{n+1}-t_n,
\qquad
n=0,\ldots,S-1.
\end{equation}
We write
\begin{equation}
\label{eq:general-grid-midpoint}
t_{n+\frac12}
=
t_n+\frac{h_n}{2}
=
\frac{t_n+t_{n+1}}{2}.
\end{equation}
The uniform grid used later is the special case
\begin{equation}
\label{eq:uniform-grid}
t_n=nh,
\qquad
h=\frac1S,
\qquad
n=0,\ldots,S.
\end{equation}
For a general nonautonomous ODE
\[
\frac{dY_t}{dt}=f(t,Y_t),
\]
the explicit midpoint approximation $Y_n^{\mathrm M}\approx Y_{t_n}$ is
defined by
\begin{equation}
\label{eq:explicit-midpoint}
\begin{aligned}
Y_{n+\frac12}^{\mathrm M}
&=Y_n^{\mathrm M}+\frac{h_n}{2}f(t_n,Y_n^{\mathrm M}),
\\
Y_{n+1}^{\mathrm M}
&=Y_n^{\mathrm M}+h_n
f(t_{n+\frac12},Y_{n+\frac12}^{\mathrm M}).
\end{aligned}
\end{equation}
Thus each interval uses two vector-field evaluations: one at its left endpoint
and one at a predicted midpoint.  Consequently, $S$ midpoint steps require
$2S$ function evaluations, so an even fixed budget of $N_{\mathrm{FE}}$
evaluations corresponds to $S=N_{\mathrm{FE}}/2$ steps.
For a fixed eigendirection $k$, define
\begin{equation}
\label{eq:directional-variance}
v_k(t)
\coloneqq
\beta_t^2+\frac{\alpha_t^2}{\lambda_k},
\qquad
A_k(t)
\coloneqq
\frac{\dot v_k(t)}{2v_k(t)}.
\end{equation}
Then the scalar ODE \eqref{eq:scalar-gaussian-ode} becomes
\begin{equation}
\label{eq:affine-scalar-ode}
\frac{dY_{t,k}^U}{dt}
=
\dot\alpha_t\mu_k^U
+A_k(t)\left(Y_{t,k}^U-\alpha_t\mu_k^U\right).
\end{equation}

\subsubsection{General midpoint recursion and output law}
\label{subsec:original-midpoint-output}

Let $Y_{n,k}^{U,\mathrm M}$ denote the midpoint approximation at time $t_n$.
Applying \eqref{eq:explicit-midpoint} to
\eqref{eq:affine-scalar-ode} and collecting the coefficients of
$Y_{n,k}^{U,\mathrm M}$ and $\mu_k^U$ yields
\begin{equation}
\label{eq:midpoint-affine-recursion}
Y_{n+1,k}^{U,\mathrm M}
=
g_k^{\mathrm M}(n)Y_{n,k}^{U,\mathrm M}
+q_k^{\mathrm M}(n)\mu_k^U,
\end{equation}
where
\begin{align}
g_k^{\mathrm M}(n)
&=
1+h_nA_k(t_{n+\frac12})
+\frac{h_n^2}{2}A_k(t_{n+\frac12})A_k(t_n),
\label{eq:midpoint-g}
\\
q_k^{\mathrm M}(n)
&=
h_n\left(
\dot\alpha_{t_{n+\frac12}}
-\alpha_{t_{n+\frac12}}A_k(t_{n+\frac12})
\right) +\frac{h_n^2}{2}A_k(t_{n+\frac12})
\left(
\dot\alpha_{t_n}-\alpha_{t_n}A_k(t_n)
\right).
\label{eq:midpoint-q}
\end{align}

\begin{proposition}
\label{prop:general-midpoint-law}
For general schedules $\alpha_t$ and $\beta_t$ on the prescribed grid
\eqref{eq:general-time-grid}, the midpoint output satisfies
\begin{equation}
\label{eq:general-midpoint-output-coordinate}
Y_{S,k}^{U,\mathrm M}
=
D_{1,k}^{\mathrm M}\epsilon_k^U
+D_{2,k}^{\mathrm M}\mu_k^U,
\end{equation}
where
\begin{align}
D_{1,k}^{\mathrm M}
&=
\prod_{n=0}^{S-1}g_k^{\mathrm M}(n),
\label{eq:D1-definition}
\\
D_{2,k}^{\mathrm M}
&=
\sum_{i=0}^{S-1}
q_k^{\mathrm M}(i)
\prod_{n=i+1}^{S-1}g_k^{\mathrm M}(n).
\label{eq:D2-definition}
\end{align}
An empty product is interpreted as $1$.  Consequently,
\begin{equation}
\label{eq:general-midpoint-output-law}
p_S^{\mathrm M}
=
\mathcal N\!\left(
U\diag(D_{2,1}^{\mathrm M},\ldots,D_{2,d}^{\mathrm M})U^\top\mu,
\Sigma_{\mathrm M}
\right),
\end{equation}
where
\begin{equation}
\label{eq:midpoint-output-covariance}
\Sigma_{\mathrm M}
=
U\diag((D_{1,1}^{\mathrm M})^2,\ldots,
(D_{1,d}^{\mathrm M})^2)U^\top.
\end{equation}
\end{proposition}

\begin{proof}
Iterating the affine recursion \eqref{eq:midpoint-affine-recursion} gives
\eqref{eq:general-midpoint-output-coordinate}.  Since
$\epsilon^U\sim\mathcal N(0,I)$, its mean and covariance give
\eqref{eq:general-midpoint-output-law}.
\end{proof}

\subsubsection{Universal mean exactness of the linear signal schedule}
\label{subsec:original-mean-proof}

The general output law in Proposition~\ref{prop:general-midpoint-law}
allows a finite-step mean error.
Proposition~\ref{prop:linear-alpha-unique} shows that the linear signal
schedule is precisely the schedule that removes the terminal mean error
for every time grid and every noise schedule satisfying the conditions
stated below.

\begin{proposition}
\label{prop:linear-alpha-unique}
Let $\alpha\in C^1([0,1])\cap C^2((0,1))$ be fixed independently of $S$, the
time grid, and $\beta$, and satisfy $\alpha_0=0$ and $\alpha_1=1$.  For any
fixed $\Sigma\succ0$, explicit midpoint has the exact terminal mean
\begin{equation}
\label{eq:universal-terminal-mean}
\mathbb E\!\left[Y_{S,k}^{U,\mathrm M}\right]=\mu_k^U
\qquad\text{for every }k
\end{equation}
for every $S\geq1$, every grid of the form
\eqref{eq:general-time-grid}, every target mean $\mu$, and every noise
schedule $\beta_t$ satisfying the conditions in Appendix A.1.1
if and only if
\begin{equation}
\alpha_t=t,
\qquad 0\leq t\leq1.
\label{eq:linear-alpha-restriction}
\end{equation}
For this schedule and every such grid and $\beta_t$, explicit midpoint in fact
reproduces the exact mean at every grid point and predicted midpoint:
\begin{equation}
\label{eq:midpoint-grid-mean}
\mathbb E\!\left[Y_{n,k}^{U,\mathrm M}\right]
=
t_n\mu_k^U,
\qquad
n=0,\ldots,S,
\end{equation}
and
\begin{equation}
\label{eq:predicted-midpoint-mean}
\mathbb E\!\left[Y_{n+\frac12,k}^{U,\mathrm M}\right]
=
t_{n+\frac12}\mu_k^U,
\qquad
n=0,\ldots,S-1.
\end{equation}
Consequently, $D_{2,k}^{\mathrm M}=1$ and
\begin{equation}
\label{eq:midpoint-output-coordinate}
Y_{S,k}^{U,\mathrm M}
=
\mu_k^U+D_{1,k}^{\mathrm M}\epsilon_k^U,
\end{equation}
so
\begin{equation}
\label{eq:midpoint-output-law}
p_S^{\mathrm M}=\mathcal N(\mu,\Sigma_{\mathrm M}).
\end{equation}
\end{proposition}

\begin{proof}
First suppose that $\alpha_t=t$.  The initial mean is
$\mathbb E[Y_{0,k}^{U,\mathrm M}]=0=t_0\mu_k^U$.  If
$\mathbb E[Y_{n,k}^{U,\mathrm M}]=t_n\mu_k^U$, then taking expectations in
the predictor and corrector stages gives
\begin{align*}
\mathbb E\!\left[Y_{n+\frac12,k}^{U,\mathrm M}\right]
&=
t_n\mu_k^U
+\frac{h_n}{2}\left[
\mu_k^U+A_k(t_n)(t_n\mu_k^U-t_n\mu_k^U)
\right]
=t_{n+\frac12}\mu_k^U,
\\
\mathbb E\!\left[Y_{n+1,k}^{U,\mathrm M}\right]
&=
t_n\mu_k^U
+h_n\left[
\mu_k^U
+A_k(t_{n+\frac12})
(t_{n+\frac12}\mu_k^U-t_{n+\frac12}\mu_k^U)
\right]
=t_{n+1}\mu_k^U.
\end{align*}
Thus induction proves \eqref{eq:midpoint-grid-mean} and
\eqref{eq:predicted-midpoint-mean}, and $t_S=1$ gives the exact terminal mean.
Proposition~\ref{prop:general-midpoint-law} also gives
\[
\mathbb E\!\left[Y_{S,k}^{U,\mathrm M}\right]
=
D_{2,k}^{\mathrm M}\mu_k^U.
\]
The coefficient $D_{2,k}^{\mathrm M}$ is independent of $\mu$, and the
preceding argument holds for every possible value of $\mu_k^U$.  Hence
$D_{2,k}^{\mathrm M}=1$.  Substitution into
\eqref{eq:general-midpoint-output-coordinate} proves
\eqref{eq:midpoint-output-coordinate}, and
\eqref{eq:midpoint-output-law} follows from
\eqref{eq:general-midpoint-output-law}.

Conversely, suppose that \eqref{eq:universal-terminal-mean} holds for every
$S$, prescribed grid, $\beta_t$, and $\mu$.  Fix an eigendirection $k$ and
take
$\mu_k^U=1$. This is allowed because
\eqref{eq:universal-terminal-mean} is assumed for every target mean.  Taking
expectations in the two stages of \eqref{eq:explicit-midpoint} gives
\begin{align}
\mathbb E\!\left[Y_{n+\frac12,k}^{U,\mathrm M}\right]
&=
\mathbb E\!\left[Y_{n,k}^{U,\mathrm M}\right]
+\frac{h_n}{2}
\left[
\dot\alpha_{t_n}
+A_k(t_n)
\left(
\mathbb E\!\left[Y_{n,k}^{U,\mathrm M}\right]-\alpha_{t_n}
\right)
\right],
\label{eq:uniqueness-mean-predictor}
\\
\mathbb E\!\left[Y_{n+1,k}^{U,\mathrm M}\right]
&=
\mathbb E\!\left[Y_{n,k}^{U,\mathrm M}\right]
+h_n
\left[
\dot\alpha_{t_{n+\frac12}}
+A_k(t_{n+\frac12})
\left(
\mathbb E\!\left[Y_{n+\frac12,k}^{U,\mathrm M}\right]
-\alpha_{t_{n+\frac12}}
\right)
\right].
\label{eq:uniqueness-mean-corrector}
\end{align}

We first record how the generality of $\beta_t$ will be used.
For this argument, it suffices to consider a subclass of the permitted
noise schedules. Fix a grid and start from $\beta_t=1-t$.
Around every interior evaluation time $s$, choose a smooth perturbation
with compact support inside $(0,1)$ that vanishes at $s$ but produces
a prescribed sufficiently small derivative change there.
The supports can be chosen disjoint and sufficiently small to contain
no other evaluation times. These perturbations therefore leave the
values of $\beta$ at all evaluation times fixed while changing its
derivatives there independently.
For sufficiently small amplitudes, $\beta$ retains the same endpoints
and remains positive for $t<1$. Together with $\alpha_1=1$, this
ensures $\alpha_t^2+\beta_t^2>0$ throughout $[0,1]$.

Moreover,
\begin{equation}
A_k(s)
=
\frac{
\beta_s\dot\beta_s+\alpha_s\dot\alpha_s/\lambda_k
}{v_k(s)},
\qquad
\delta A_k(s)
=
\frac{\beta_s}{v_k(s)}\,\delta\dot\beta_s.
\label{eq:A-variation-for-uniqueness}
\end{equation}
At each interior evaluation time $s$, the perturbations preserve
$\beta_s=1-s>0$, and hence $v_k(s)\geq\beta_s^2>0$.
Thus the coefficient $\beta_s/v_k(s)$ in the second identity of
\eqref{eq:A-variation-for-uniqueness} is nonzero, so a local change
in $\dot\beta_s$ produces any sufficiently small change in $A_k(s)$.
Since the perturbations have disjoint supports, these changes can
be made independently at all sampled interior evaluation times.

We now argue backward through the midpoint steps.  Suppose, for some $n$, that the following identity remains true when each of
the sampled values
\[
A_k(t_1),\ldots,A_k(t_{S-1}),
\qquad
A_k(t_{\frac12}),\ldots,A_k(t_{S-\frac12})
\]
is varied independently by a sufficiently small amount:
\begin{equation}
\mathbb E\!\left[Y_{n+1,k}^{U,\mathrm M}\right]
=
\alpha_{t_{n+1}}.
\label{eq:backward-exact-endpoint}
\end{equation}
For $n=S-1$, \eqref{eq:backward-exact-endpoint} is exactly the terminal-mean
condition \eqref{eq:universal-terminal-mean}.  In
\eqref{eq:uniqueness-mean-corrector}, neither
$\mathbb E[Y_{n,k}^{U,\mathrm M}]$ nor the predicted midpoint mean depends on
$A_k(t_{n+\frac12})$.  Since the left-hand side must remain equal to
$\alpha_{t_{n+1}}$ as $A_k(t_{n+\frac12})$ varies, its coefficient must vanish:
\begin{equation}
\mathbb E\!\left[Y_{n+\frac12,k}^{U,\mathrm M}\right]
=
\alpha_{t_{n+\frac12}}.
\label{eq:backward-exact-midpoint}
\end{equation}
The part of the corrector that is independent of $A_k(t_{n+\frac12})$ also gives
\begin{equation}
\alpha_{t_{n+1}}
=
\mathbb E\!\left[Y_{n,k}^{U,\mathrm M}\right]
+h_n\dot\alpha_{t_{n+\frac12}}.
\label{eq:backward-corrector-constant}
\end{equation}

For $n\geq1$, the expectation
$\mathbb E[Y_{n,k}^{U,\mathrm M}]$ depends only on earlier vector-field
evaluations and is therefore independent of $A_k(t_n)$.  Substituting
\eqref{eq:backward-exact-midpoint} into
\eqref{eq:uniqueness-mean-predictor} and varying $A_k(t_n)$ forces
\begin{equation}
\mathbb E\!\left[Y_{n,k}^{U,\mathrm M}\right]
=
\alpha_{t_n}.
\label{eq:backward-exact-gridpoint}
\end{equation}
We have therefore shown that exactness at the right endpoint of step $n$,
namely
\[
\mathbb E\!\left[Y_{n+1,k}^{U,\mathrm M}\right]=\alpha_{t_{n+1}},
\]
implies exactness at its left endpoint:
\[
\mathbb E\!\left[Y_{n,k}^{U,\mathrm M}\right]=\alpha_{t_n}.
\]
This left endpoint is also the right endpoint of step $n-1$.  Hence
\eqref{eq:backward-exact-gridpoint} is precisely the hypothesis needed to
repeat the same argument with $n$ replaced by $n-1$.  Continuing in this way
propagates exactness backward from $t_S$ to $t_0$.
For $n=0$, \eqref{eq:backward-exact-gridpoint} follows directly from the
initial condition:
$\mathbb E[Y_{0,k}^{U,\mathrm M}]=0=\alpha_0$.  The parts of
\eqref{eq:uniqueness-mean-predictor} and
\eqref{eq:backward-corrector-constant} that remain after these residuals
vanish now give
\begin{equation}
\label{eq:necessary-alpha-grid-identities}
\alpha_{t_{n+\frac12}}
=
\alpha_{t_n}+\frac{h_n}{2}\dot\alpha_{t_n},
\qquad
\alpha_{t_{n+1}}
=
\alpha_{t_n}+h_n\dot\alpha_{t_{n+\frac12}}.
\end{equation}
Thus exactness at $t_{n+1}$ implies exactness at $t_n$ and at the predicted
midpoint.  Starting from $t_S=1$ and repeating this argument backward proves
\eqref{eq:necessary-alpha-grid-identities} on every interval of every grid.
\\
\noindent Since the assumed mean exactness holds for every prescribed grid,
\eqref{eq:necessary-alpha-grid-identities} holds in particular on the uniform
grids \eqref{eq:uniform-grid}.  Applying its first identity on these grids
finishes the argument.  For every $n=1,\ldots,S-1$, Taylor's theorem gives a
point $\xi_{n,S}\in(t_n,t_{n+\frac12})$ such that
\[
\alpha_{t_{n+\frac12}}
=
\alpha_{t_n}
+\frac h2\dot\alpha_{t_n}
+\frac{h^2}{8}\ddot\alpha_{\xi_{n,S}}.
\]
Comparison with \eqref{eq:necessary-alpha-grid-identities} shows that
$\ddot\alpha_{\xi_{n,S}}=0$.  Given any $t\in(0,1)$, for all sufficiently fine
uniform grids the cell containing $t$ is an interior cell and contains such a
point $\xi_{n,S}$.  Since $|\xi_{n,S}-t|\leq h\to0$, continuity gives
$\ddot\alpha_t=0$.  Hence $\alpha$ is affine, and the endpoint conditions
force $\alpha_t=t$.
\end{proof} 
\noindent From this point onward, we fix $\alpha_t=t$ and design only the noise schedule
$\beta_t$.

\subsection{The exact finite-step objective}
\label{sec:objectives}
\label{app:computation}

Subsection~\ref{sec:midpoint-analysis} produced the exact Gaussian law of the
explicit-midpoint sampler under $\alpha_t=t$.  Schedule design now becomes an
exact comparison between $p_S^{\mathrm M}$ in
\eqref{eq:midpoint-output-law} and the target Gaussian law in
\eqref{eq:gaussian-data}.

\noindent The midpoint output and target distribution have the same mean $\mu$.
Consequently, the squared $2$-Wasserstein distance between them is
\begin{align*}
W_2^2(p_S^{\mathrm M},p_{\mathrm{data}})
&=
\operatorname{tr}
\left(
\Sigma_{\mathrm M}+\Sigma
-2(\Sigma^{1/2}\Sigma_{\mathrm M}\Sigma^{1/2})^{1/2}
\right).
\end{align*}
Moreover, $\Sigma$ and $\Sigma_{\mathrm M}$ share the eigenbasis $U$.
Their directional standard deviations are, respectively,
$\lambda_k^{-1/2}$ and $|D_{1,k}^{\mathrm M}|$.  Therefore the distance
reduces exactly to
\begin{equation}
\boxed{
\label{eq:w2-objective}
J_S
\coloneqq
W_2^2(p_S^{\mathrm M},p_{\mathrm{data}})
=
\sum_{k=1}^d
\left(
\left|D_{1,k}^{\mathrm M}\right|
-\frac{1}{\sqrt{\lambda_k}}
\right)^2.
}
\end{equation}

\section{CondOT Superconvergence}
\label{sec:condot-midpoint-asymptotics}

This appendix provides the general error expansion used in
Section~\ref{m:analysis}, followed by the proofs of
Theorem~\ref{m:condot} and Corollary~\ref{m:condotuniformexpansion}.

Except when a particular schedule is substituted explicitly,
$\beta_t$ denotes any fixed function in $C^1([0,1])$ with
$\beta_0=1$ and $\beta_1=0$ that satisfies the additional regularity
assumptions stated below. Here, fixed means that the schedule is
independent of the grid as it is refined.
For the prescribed grid \eqref{eq:general-time-grid}, define its mesh
size by
\begin{equation}
\label{eq:mesh-size}
h_{\max}
\coloneqq
\max_{0\leq n\leq S-1}h_n.
\end{equation}
All general-grid asymptotic statements take $h_{\max}\to0$. We later specialize to the uniform grid,
for which $h_{\max}=h=S^{-1}$.

\subsection{General-grid midpoint multiplier and logarithmic scale error}
\label{subsec:midpoint-positivity}
Under the standing restriction $\alpha_t=t$, fix an eigendirection $k$.  Then
\begin{equation}
\label{eq:mean-exact-directional-variance}
v_k(t)=\beta_t^2+\frac{t^2}{\lambda_k},
\qquad
A_k(t)=\frac{\dot v_k(t)}{2v_k(t)}.
\end{equation}
The one-step multiplier is
\begin{equation}
\label{eq:general-midpoint-factor}
g_k^{\mathrm M}(n)
=
1+h_nA_k(t_{n+\frac12})
+\frac{h_n^2}{2}
A_k(t_{n+\frac12})A_k(t_n).
\end{equation}
The target standard deviation in direction $k$ is $\lambda_k^{-1/2}$, while
the midpoint sampler output has standard deviation $|D_{1,k}^{\mathrm M}|$.
For the high-resolution analysis below, assume that
$v_k\in C^4([0,1])$ and that there exists $c_k>0$ such that
$v_k(t)\geq c_k$ for every $t\in[0,1]$.
Then $A_k$ and its first three derivatives are bounded.
Since $h_n\leq h_{\max}$, for all sufficiently small $h_{\max}$, every
multiplier $g_k^{\mathrm M}(n)$ is positive, uniformly in $n$, and hence
$D_{1,k}^{\mathrm M}>0$.
The ratio $\sqrt{\lambda_k}D_{1,k}^{\mathrm M}$ is then the numerical scale
divided by the target scale, and its logarithmic relative error is
\begin{equation}
\label{eq:general-grid-log-scale-error}
e_k
\coloneqq
\log\!\left(\sqrt{\lambda_k}D_{1,k}^{\mathrm M}\right).
\end{equation}
It is zero exactly when the terminal scale is correct, and its sign records
underestimation or overestimation.  This is not a new objective.  Since
\[
D_{1,k}^{\mathrm M}-\frac1{\sqrt{\lambda_k}}
=
\frac{e^{e_k}-1}{\sqrt{\lambda_k}},
\]
and $e^x-1=x+\frac{x^2}{2}+O(|x|^3)$, we obtain
\begin{equation}
\label{eq:general-grid-log-error-to-w2}
\left(
D_{1,k}^{\mathrm M}-\frac1{\sqrt{\lambda_k}}
\right)^2
=
\frac{e_k^2+e_k^3+O(e_k^4)}{\lambda_k}
=
\frac{e_k^2}{\lambda_k}
+O\!\left(|e_k|^3\right).
\end{equation}

\noindent Thus the logarithmic error determines the directional squared Wasserstein
scale error in \eqref{eq:w2-objective}.  At coarse resolution the exact
objective continues to use $|D_{1,k}^{\mathrm M}|$, even if multiplier
positivity is not guaranteed.
Since $v_k(0)=1$ and $v_k(1)=\lambda_k^{-1}$,
\begin{equation}
\label{eq:exact-log-multiplier}
\int_0^1A_k(t)\,dt
=
\frac12[\log v_k(t)]_0^1
=
-\frac12\log\lambda_k.
\end{equation}
Using $D_{1,k}^{\mathrm M}=\prod_{n=0}^{S-1}g_k^{\mathrm M}(n)$ gives
\begin{align}
e_k
&=
\sum_{n=0}^{S-1}\log g_k^{\mathrm M}(n)
-\int_0^1A_k(t)\,dt =
\sum_{n=0}^{S-1}
\left[
\log g_k^{\mathrm M}(n)
-\int_{t_n}^{t_{n+1}}A_k(t)\,dt
\right].
\label{eq:general-grid-global-error-local-defects}
\end{align}

\subsection{The leading midpoint defect on a general grid}
\label{subsec:original-local-defect}

\begin{proposition}
\label{prop:local-log-defect}
Suppose $v_k\in C^4([0,1])$ and $v_k(t)\geq c_k>0$.  For each step of a
prescribed grid, as $h_{\max}\to0$,
\begin{align}
\log g_k^{\mathrm M}(n)
-\int_{t_n}^{t_{n+1}}A_k(t)\,dt
&=
-\frac{h_n^3}{24}
\left(A_k''+6A_kA_k'+4A_k^3\right)(t_n)
+O(h_n^4)
\label{eq:local-log-defect-A}
\\
&=
-\frac{h_n^3}{48}
\frac{v_k'''(t_n)}{v_k(t_n)}
+O(h_n^4),
\label{eq:local-log-defect-v}
\end{align}
where, for each fixed schedule and eigendirection $k$, the $O(h_n^4)$ term is
bounded by $C h_n^4$ with the same constant $C$ for every step and every
sufficiently fine prescribed grid.
\end{proposition}

\begin{proof}
Taylor expanding around $t_n$ gives
\[
A_k(t_{n+\frac12})
=
A_k(t_n)
+\frac{h_n}{2}A_k'(t_n)
+\frac{h_n^2}{8}A_k''(t_n)
+O(h_n^3).
\]
Substituting this into \eqref{eq:general-midpoint-factor} yields
\begin{align*}
g_k^{\mathrm M}(n)
&=
1
+h_n\left(A_k+\frac{h_n}{2}A_k'+\frac{h_n^2}{8}A_k''\right)(t_n)
+\frac{h_n^2}{2}
\left(A_k+\frac{h_n}{2}A_k'\right)(t_n)A_k(t_n)
+O(h_n^4)
\\
&=
1+h_nA_k(t_n)
+\frac{h_n^2}{2}(A_k'+A_k^2)(t_n)
+h_n^3\left(\frac18A_k''+\frac14A_kA_k'\right)(t_n)
+O(h_n^4).
\end{align*}

\noindent Set $x=g_k^{\mathrm M}(n)-1$.  The preceding expansion gives
\[
x^2
=
h_n^2A_k^2(t_n)
+h_n^3(A_kA_k'+A_k^3)(t_n)
+O(h_n^4),
\qquad
x^3=h_n^3A_k^3(t_n)+O(h_n^4).
\]
Therefore, using $\log(1+x)=x-x^2/2+x^3/3+O(x^4)$,
\begin{align}
\log g_k^{\mathrm M}(n)
&=
h_nA_k(t_n)+\frac{h_n^2}{2}A_k'(t_n)
\notag
\\
&\quad
+h_n^3\left(
\frac18A_k''
+\frac14A_kA_k'
-\frac12(A_kA_k'+A_k^3)
+\frac13A_k^3
\right)(t_n)
+O(h_n^4)
\notag
\\
&=
h_nA_k(t_n)+\frac{h_n^2}{2}A_k'(t_n)
+h_n^3\left(
\frac18A_k''-\frac14A_kA_k'-\frac16A_k^3
\right)(t_n)
+O(h_n^4).
\label{eq:log-factor-expansion}
\end{align}
The exact integral satisfies
\begin{equation}
\int_{t_n}^{t_{n+1}}A_k(t)\,dt
=
h_nA_k(t_n)+\frac{h_n^2}{2}A_k'(t_n)
+\frac{h_n^3}{6}A_k''(t_n)+O(h_n^4).
\label{eq:integral-expansion}
\end{equation}
Subtracting \eqref{eq:integral-expansion} from
\eqref{eq:log-factor-expansion} gives
\eqref{eq:local-log-defect-A}.  Finally, differentiating
$v_k'=2A_kv_k$ twice gives
\[
v_k'''
=
2v_k(A_k''+6A_kA_k'+4A_k^3).
\]
This proves \eqref{eq:local-log-defect-v}.
\end{proof}
\noindent Summing Proposition~\ref{prop:local-log-defect} over the prescribed
grid gives
\begin{equation}
\label{eq:general-grid-global-log-error}
e_k
=
-\frac1{48}
\sum_{n=0}^{S-1}
h_n^3\frac{v_k'''(t_n)}{v_k(t_n)}
+O\!\left(\sum_{n=0}^{S-1}h_n^4\right).
\end{equation}
Because $\sum_nh_n=1$ and $h_n\leq h_{\max}$,
\begin{equation}
\label{eq:mesh-power-bounds}
\sum_{n=0}^{S-1}h_n^3\leq h_{\max}^2,
\qquad
\sum_{n=0}^{S-1}h_n^4\leq h_{\max}^3.
\end{equation}
Hence $e_k=O(h_{\max}^2)$.  Combining
\eqref{eq:general-grid-global-log-error} with
\eqref{eq:general-grid-log-error-to-w2} and summing over the eigendirections
gives the sharper discrete expansion
\begin{equation}
\label{eq:general-grid-global-w2-error}
\boxed{
\begin{aligned}
W_2^2(p_S^{\mathrm M},p_{\mathrm{data}})
&=
\frac1{48^2}
\sum_{k=1}^d\frac1{\lambda_k}
\left[
\sum_{n=0}^{S-1}
h_n^3\frac{v_k'''(t_n)}{v_k(t_n)}
\right]^2
+O(h_{\max}^5)
=O(h_{\max}^4).
\end{aligned}
}
\end{equation}

\subsection{CondOT cancellation on a general grid}
\label{subsec:original-condot-general}

For CondOT,
\begin{equation}
\label{eq:condot-path}
\alpha_t=t,
\qquad
\beta_t=1-t,
\end{equation}
and therefore
\begin{equation}
\label{eq:condot-variance}
v_k(t)
=
(1-t)^2+\frac{t^2}{\lambda_k}.
\end{equation}
This variance is a positive quadratic, so $v_k'''\equiv0$.  Hence the complete
leading defect in Proposition~\ref{prop:local-log-defect} vanishes in every
eigendirection and for every $\lambda_k>0$.

\begin{proposition}
\label{prop:general-grid-condot-superconvergence}
For CondOT, the local logarithmic defect on every step of a prescribed grid
satisfies
\begin{equation}
\label{eq:condot-general-grid-local-defect}
\log g_k^{\mathrm M}(n)
-\int_{t_n}^{t_{n+1}}A_k(t)\,dt
=
-\frac{h_n^4}{8}
\left(A_k^2A_k'+A_k^4\right)(t_n)
+O(h_n^5),
\end{equation}
where the remainder constant is independent of $n$ and of the prescribed
grid.  Consequently,
\begin{equation}
\label{eq:condot-general-grid-log-error}
e_k
=
-\frac18\sum_{n=0}^{S-1}h_n^4
\left(A_k^2A_k'+A_k^4\right)(t_n)
+O\!\left(\sum_{n=0}^{S-1}h_n^5\right)
=O(h_{\max}^3),
\end{equation}
and
\begin{equation}
\label{eq:condot-general-grid-scale-error}
D_{1,k}^{\mathrm M}-\frac1{\sqrt{\lambda_k}}
=
-\frac1{8\sqrt{\lambda_k}}
\sum_{n=0}^{S-1}h_n^4
\left(A_k^2A_k'+A_k^4\right)(t_n)
+O(h_{\max}^4).
\end{equation}
The global squared Wasserstein error therefore obeys
\begin{equation}
\label{eq:condot-general-grid-w2}
\boxed{
\begin{aligned}
W_2^2(p_S^{\mathrm M},p_{\mathrm{data}})
&=
\frac1{64}\sum_{k=1}^d\frac1{\lambda_k}
\left[
\sum_{n=0}^{S-1}h_n^4
\left(A_k^2A_k'+A_k^4\right)(t_n)
\right]^2
+O(h_{\max}^7)=O(h_{\max}^6).
\end{aligned}
}
\end{equation}
Thus CondOT improves the logarithmic error by one power of $h_{\max}$ and the
squared Wasserstein error by two powers on every refining prescribed grid.
\end{proposition}

\begin{proof}
Because the CondOT variance in \eqref{eq:condot-variance} is positive and
quadratic, $A_k=v_k'/(2v_k)$ is smooth on $[0,1]$.  Fix a step, abbreviate
$A=A_k(t_n)$ and likewise for its derivatives, and expand
\[
A_k(t_{n+\frac12})
=
A+\frac{h_n}{2}A'
+\frac{h_n^2}{8}A''
+\frac{h_n^3}{48}A'''
+O(h_n^4).
\]
Substitution into \eqref{eq:general-midpoint-factor} gives
\begin{align*}
g_k^{\mathrm M}(n)
&=
1+h_nA
+\frac{h_n^2}{2}(A'+A^2)
+h_n^3\left(\frac18A''+\frac14AA'\right)
\\
&\quad
+h_n^4\left(\frac1{48}A'''+\frac1{16}AA''\right)
+O(h_n^5).
\end{align*}
Using
$\log(1+x)=x-x^2/2+x^3/3-x^4/4+O(x^5)$ and retaining terms through order
$h_n^4$ yields
\begin{align*}
\log g_k^{\mathrm M}(n)
&=
h_nA+\frac{h_n^2}{2}A'
+h_n^3\left(
\frac18A''-\frac14AA'-\frac16A^3
\right)
\\
&\quad
+h_n^4\left(
\frac1{48}A'''
-\frac1{16}AA''
-\frac18(A')^2
+\frac18A^4
\right)
+O(h_n^5),
\end{align*}
whereas
\[
\int_{t_n}^{t_{n+1}}A_k(t)\,dt
=
h_nA+\frac{h_n^2}{2}A'
+\frac{h_n^3}{6}A''
+\frac{h_n^4}{24}A'''
+O(h_n^5).
\]
Since $v_k'''=0$ and
\[
v_k'''
=
2v_k\left(A_k''+6A_kA_k'+4A_k^3\right),
\]
we have
\[
A_k''+6A_kA_k'+4A_k^3=0.
\]
Differentiating this identity gives
\[
A_k'''
+6(A_k')^2
+6A_kA_k''
+12A_k^2A_k'
=0.
\]
Subtracting the integral expansion from the logarithmic expansion and using
these two identities reduces the fourth-order coefficient to
$-\tfrac18(A_k^2A_k'+A_k^4)$, proving
\eqref{eq:condot-general-grid-local-defect}.

\noindent Summing the local defects gives the first equality in
\eqref{eq:condot-general-grid-log-error}.  The bounds
\[
\sum_{n=0}^{S-1}h_n^4\leq h_{\max}^3,
\qquad
\sum_{n=0}^{S-1}h_n^5\leq h_{\max}^4
\]
give its final estimate.  Finally,
\[
D_{1,k}^{\mathrm M}-\frac1{\sqrt{\lambda_k}}
=
\frac{\exp(e_k)-1}{\sqrt{\lambda_k}}
=
\frac{e_k}{\sqrt{\lambda_k}}+O(e_k^2)
\]
together with \eqref{eq:condot-general-grid-log-error} proves
\eqref{eq:condot-general-grid-scale-error}.  Its leading sum is
$O(h_{\max}^3)$ and its remainder is $O(h_{\max}^4)$, so squaring produces an
$O(h_{\max}^7)$ cross term.  Summing over the eigendirections gives
\eqref{eq:condot-general-grid-w2}.  In addition, direct differentiation of
\eqref{eq:condot-variance} gives
\[
A_k'+A_k^2=\frac1{\lambda_kv_k^2},
\qquad
A_k^2A_k'+A_k^4
=\frac{A_k^2}{\lambda_kv_k^2}\geq0,
\]
so the leading logarithmic bias is nonpositive on every prescribed grid.
\end{proof}

\subsection{Uniform-grid generic asymptotics}

We now specialize to the uniform grid \eqref{eq:uniform-grid}, for which
$h_n=h=S^{-1}$ and $h_{\max}=h$.  The one-step multiplier becomes
\begin{equation}
\label{eq:uniform-midpoint-factor}
g_k^{\mathrm M}(n)
=
1+hA_k(t_{n+\frac12})
+\frac{h^2}{2}
A_k(t_{n+\frac12})A_k(t_n).
\end{equation}
We write the logarithmic error on this grid as
\begin{equation}
\label{eq:log-scale-error}
e_k(h)
\coloneqq
e_k
=
\log\!\left(\sqrt{\lambda_k}D_{1,k}^{\mathrm M}\right).
\end{equation}
Equation~\eqref{eq:general-grid-log-error-to-w2} then becomes
\begin{equation}
\label{eq:log-error-to-w2}
\left(
D_{1,k}^{\mathrm M}-\frac1{\sqrt{\lambda_k}}
\right)^2
=
\frac{e_k(h)^2+e_k(h)^3+O(e_k(h)^4)}{\lambda_k}
=
\frac{e_k(h)^2}{\lambda_k}
+O\!\left(|e_k(h)|^3\right).
\end{equation}
Moreover, \eqref{eq:general-grid-global-error-local-defects} specializes to
\begin{align}
e_k(h)
&=
\sum_{n=0}^{S-1}\log g_k^{\mathrm M}(n)
-\int_0^1A_k(t)\,dt
\notag
\\
&=
\sum_{n=0}^{S-1}
\left[
\log g_k^{\mathrm M}(n)
-\int_{t_n}^{t_n+h}A_k(t)\,dt
\right].
\label{eq:global-error-local-defects}
\end{align}
Summing Proposition~\ref{prop:local-log-defect} over the $S=h^{-1}$ steps
gives
\begin{equation}
e_k(h)
=
-\frac{h^2}{48}
\left[
h\sum_{n=0}^{S-1}
\frac{v_k'''(t_n)}{v_k(t_n)}
\right]
+O(h^3),
\label{eq:uniform-global-error-sum}
\end{equation}
because the local remainders sum to $S\,O(h^4)=O(h^3)$.  Set
$f_k=v_k'''/v_k$.  Under the stated assumptions, $f_k\in C^1([0,1])$, and
the mean value theorem gives
\[
\left|
h\sum_{n=0}^{S-1}f_k(t_n)-\int_0^1f_k(t)\,dt
\right|
\leq
\frac12\|f_k'\|_\infty Sh^2
=
O(h).
\]
Thus the bracketed term equals the integral plus $O(h)$.  After
multiplication by $h^2$, this contributes another $O(h^3)$ term, and hence
\begin{equation}
\label{eq:generic-global-log-error}
e_k(h)
=
-\frac{h^2}{48}
\int_0^1\frac{v_k'''(t)}{v_k(t)}\,dt
+O(h^3).
\end{equation}
Combining this with \eqref{eq:log-error-to-w2} and summing over the
eigendirections gives
\begin{equation}
\label{eq:generic-global-w2-error}
\boxed{
\begin{aligned}
W_2^2(p_S^{\mathrm M},p_{\mathrm{data}})=
\frac{h^4}{48^2}
\sum_{k=1}^d
\frac{1}{\lambda_k}
\left(
\int_0^1\frac{v_k'''(t)}{v_k(t)}\,dt
\right)^2
+O(h^5).
\end{aligned}
}
\end{equation}
The expansions \eqref{eq:generic-global-log-error} and
\eqref{eq:generic-global-w2-error} apply to any fixed noise schedule for which
$v_k\in C^4([0,1])$ and $v_k$ is bounded away from zero.

\paragraph{The four fixed comparison schedules.}
\label{app:squared-noise-baseline}
The four comparison schedules in Section~\ref{m:experiments} are smooth,
nonincreasing, and positive for $t<1$, with the required endpoints.
For each fixed $\lambda_k>0$, the continuous variance
$v_k(t)=\beta_t^2+t^2/\lambda_k$ is positive on $[0,1]$ and therefore
bounded away from zero. Their third variance derivatives are
\[
\begin{array}{c|c}
 \beta_t & v_k'''(t)\\ \hline
 (1-t)^2 & -24(1-t)\\
 1-t^2 & 24t\\
 \cos(\pi t/2) & (\pi^3/2)\sin(\pi t)\\
 (1-t)(1+0.1t) & 1.08+0.24t
\end{array}
\]
For each schedule, $v_k'''$ has a strict, constant sign on $(0,1)$.
Thus $\int_0^1v_k'''(t)/v_k(t)\,dt\ne0$ for every positive precision.
Equation~\eqref{eq:generic-global-w2-error} then has a strictly positive
leading coefficient, proving $J_S=\Theta(S^{-4})$ for every fixed
positive precision spectrum. 

\begin{samepage}
\subsection{CondOT midpoint superconvergence on the uniform grid}
\label{subsec:original-condot-uniform}

We now specialize the CondOT cancellation to the uniform grid in order to
recover a grid-independent leading coefficient.

\begin{corollary}
\label{cor:condot-superconvergence}
For CondOT on the uniform grid \eqref{eq:uniform-grid},
\begin{equation}
\label{eq:condot-log-error}
e_k(h)
=
d_kh^3+O(h^4),
\end{equation}
where
\begin{align}
\label{eq:condot-bias}
d_k
&=
-\frac18
\left(
\frac23+\int_0^1A_k(t)^4\,dt
\right) =
-\frac1{24}
-\frac{\pi(1+\lambda_k)^3}{256\lambda_k^{3/2}}
<0.
\end{align}
Consequently,
\begin{align}
D_{1,k}^{\mathrm M}-\frac1{\sqrt{\lambda_k}}
&=
\frac{d_k}{\sqrt{\lambda_k}}h^3+O(h^4),
\label{eq:condot-scale-error}
\\
\left(
D_{1,k}^{\mathrm M}-\frac1{\sqrt{\lambda_k}}
\right)^2
&=
\frac{d_k^2}{\lambda_k}h^6+O(h^7).
\label{eq:condot-squared-scale-error}
\end{align}
Summing over the eigendirections gives
\begin{equation}
\label{eq:condot-global-w2}
\boxed{
W_2^2(p_S^{\mathrm M},p_{\mathrm{data}})
=
h^6\sum_{k=1}^d\frac{d_k^2}{\lambda_k}
+O(h^7).
}
\end{equation}
\end{corollary}
\end{samepage}

\begin{proof}
On the uniform grid, \eqref{eq:condot-general-grid-log-error} becomes
\begin{align*}
e_k(h)
&=
-\frac{h^3}{8}
\left[
h\sum_{n=0}^{S-1}
\left(A_k^2A_k'+A_k^4\right)(t_n)
\right]
+O(h^4)
\\
&=
-\frac{h^3}{8}
\int_0^1
\left(A_k^2A_k'+A_k^4\right)(t)\,dt
+O(h^4),
\end{align*}
where the bracketed expression is a left Riemann sum of a smooth function.

\noindent Under \eqref{eq:condot-variance},
\[
A_k(t)
=
\frac{t/\lambda_k-(1-t)}
{(1-t)^2+t^2/\lambda_k}.
\]
Thus $A_k(0)=-1$ and $A_k(1)=1$, so
\[
\int_0^1A_k(t)^2A_k'(t)\,dt
=
\left[\frac{A_k(t)^3}{3}\right]_0^1
=
\frac23.
\]
For the fourth-power term, set
\[
y
=
\frac{(1+\lambda_k)t-\lambda_k}{\sqrt{\lambda_k}}.
\]
Then
\[
A_k(t)
=
\frac{(1+\lambda_k)y}
{\sqrt{\lambda_k}(1+y^2)},
\qquad
dt
=
\frac{\sqrt{\lambda_k}}{1+\lambda_k}\,dy,
\]
and the endpoints $t=0,1$ correspond to
$y=-\sqrt{\lambda_k}$ and $y=1/\sqrt{\lambda_k}$.  Hence
\begin{align*}
\int_0^1A_k(t)^4\,dt
&=
\frac{(1+\lambda_k)^3}{\lambda_k^{3/2}}
\int_{-\sqrt{\lambda_k}}^{1/\sqrt{\lambda_k}}
\frac{y^4}{(1+y^2)^4}\,dy
\\
&=
\frac{\pi(1+\lambda_k)^3}{32\lambda_k^{3/2}}
-\frac13,
\end{align*}
where the last integral is evaluated using $y=\tan\varphi$.

\noindent Substituting the two integrals into the global expansion gives
\[
e_k(h)
=
-\frac18
\left(
\frac23+\int_0^1A_k(t)^4\,dt
\right)h^3
+O(h^4)
=
d_kh^3+O(h^4),
\]
with
\[
d_k
=
-\frac1{24}
-\frac{\pi(1+\lambda_k)^3}{256\lambda_k^{3/2}}
<0.
\]
This proves \eqref{eq:condot-log-error} and
\eqref{eq:condot-bias}.

\noindent Finally, by the definition of $e_k(h)$,
\[
D_{1,k}^{\mathrm M}
=
\frac{\exp(e_k(h))}{\sqrt{\lambda_k}}.
\]
Since $e_k(h)=d_kh^3+O(h^4)$,
\[
D_{1,k}^{\mathrm M}
-\frac1{\sqrt{\lambda_k}}
=
\frac{\exp(e_k(h))-1}{\sqrt{\lambda_k}}
=
\frac{d_k}{\sqrt{\lambda_k}}h^3+O(h^4).
\]
Squaring this expression and summing the exact directional objective
\eqref{eq:w2-objective} proves
\eqref{eq:condot-scale-error}--\eqref{eq:condot-global-w2}.
\end{proof}

\section{Exact finite-budget calibration near CondOT}
\label{app:exact-calibration}

This appendix proves Theorem~\ref{m:exact-calibration}. We keep the
Gaussian target fixed, use its exact vector field, and take
$\alpha_t=t$ on a uniform grid with step size $1/S$.
We first introduce the correction family and the matrix describing
how its coefficients affect the sampling errors, and prove that
this matrix is invertible. We then use two error estimates to
construct the leading correction and prove exact calibration.
The derivation of these estimates, including the formula for
the matrix, is given at the end of
Appendix~\ref{app:calibration-uniform-expansion}.

\subsection{The correction family and its linear system}
\label{app:calibration-family}

As in Section~\ref{m:calibration}, let
$\lambda_1,\ldots,\lambda_r$ be the distinct positive precision
eigenvalues. Repeated eigenvalues give identical scalar dynamics, so
one equation for each distinct value suffices. The proposed family is
\begin{equation}
 \beta_{\boldsymbol\theta}(t)
 =(1-t)\left(1+\sum_{j=1}^r\theta_jt^j\right).
 \label{eq:calibration-family}
\end{equation}
It is one scalar noise schedule shared by all eigendirections.
At $\boldsymbol\theta=0$ it is CondOT, and for every choice of
coefficients it has the required endpoints
$\beta_{\boldsymbol\theta}(0)=1$ and
$\beta_{\boldsymbol\theta}(1)=0$.
Being polynomial, it is smooth for every coefficient vector.
Since $\alpha_t=t$, the nondegeneracy condition also holds for
every coefficient vector: at $t=0$,
$\beta_{\boldsymbol\theta}(0)=1$, and for $t>0$,
$t^2+\beta_{\boldsymbol\theta}(t)^2>0$.

\paragraph{The matrix in the leading correction equations.}
We use one matrix to write the linear system compactly:
\begin{equation}
 L_{kj}=-\frac1{48}\int_0^1
 \frac{\big[2t^j(1-t)^2\big]'''}
 {(1-t)^2+t^2/\lambda_k}\,dt,
 \qquad 1\leq k,j\leq r.
 \label{eq:calibration-response}
\end{equation}
Primes here denote derivatives with respect to $t$.
The coefficients already computed for CondOT are
\begin{equation}
 d_k=-\frac1{24}
 -\frac{\pi(1+\lambda_k)^3}{256\lambda_k^{3/2}}<0.
 \label{eq:calibration-condot-coefficient}
\end{equation}
The equations in~\eqref{m:calibration-linear-system} are therefore
$L\boldsymbol K=-(d_1,\ldots,d_r)^\top$.
The factor $-1/48$ in $L$ accounts for the sign and factor $48$ in the
main-text equations. We next show that this system always has a unique
solution for a fixed set of distinct positive precisions.

\begin{lemma}
\label{lem:calibration-response-invertible}
For every finite set of distinct positive precisions,
the matrix $L$ in~\eqref{eq:calibration-response} is invertible.
\end{lemma}

\begin{proof}
We argue by contradiction. If the square matrix $L$ were singular,
there would be a nonzero vector
$\boldsymbol a=(a_1,\ldots,a_r)^\top$ such that
$\boldsymbol a^\top L=0$. In other words,
$\sum_{k=1}^r a_kL_{kj}=0$ for every column $j$.
Define the continuous function
\[
 W(t)=\sum_{k=1}^r
 \frac{a_k}{(1-t)^2+t^2/\lambda_k}.
\]
Each denominator is positive on $[0,1]$ because $\lambda_k>0$ and
$t$ and $1-t$ cannot both be zero.

\noindent\textbf{Step 1: The null-vector equations imply orthogonality.}
Substituting~\eqref{eq:calibration-response} and moving the finite sum
inside the integral gives, for $j=1,\ldots,r$,
\[
 0=-48\sum_{k=1}^r a_kL_{kj}
   =\int_0^1W(t)\big[2t^j(1-t)^2\big]'''\,dt.
\]
The polynomial $[2t^j(1-t)^2]'''$ has degree $j-1$ and leading
coefficient $2j(j+1)(j+2)$, which is nonzero.
Indeed, the highest power before differentiating is $2t^{j+2}$.
Consequently, these $r$ polynomials have respective degrees
$0,1,\ldots,r-1$ and are linearly independent: in a vanishing linear
combination, the coefficient of $t^{r-1}$ first forces the coefficient
of the last polynomial to be zero; the coefficient of $t^{r-2}$ then
forces the preceding one to be zero, and so on.
The space of polynomials of degree at most $r-1$ has dimension $r$,
so these polynomials form a basis of that space.

Every polynomial $p$ of degree at most $r-1$ is thus a linear
combination of them. By linearity of integration, the preceding
equalities imply
\[
 \int_0^1W(t)p(t)\,dt=0
 \qquad\text{for every polynomial }p\text{ with }\deg p\leq r-1.
\]

We will construct one such polynomial for which this integral is
strictly positive, giving a contradiction.

\noindent\textbf{Step 2: The function $W$ has at most $r-1$ sign changes.}
We now rewrite the same function $W$ as a single fraction. This will
allow us to count its zeros by examining the polynomial in its numerator.
For $0<t<1$, we can factor $t^2$ out of each denominator:
\[
 (1-t)^2+\frac{t^2}{\lambda_k}
 =t^2\left[\left(\frac{1-t}{t}\right)^2+\frac1{\lambda_k}\right].
\]
Substituting this identity into the original definition of $W$ and
taking the common factor $1/t^2$ outside the sum gives
\[
 W(t)=\frac1{t^2}\sum_{k=1}^r
 \frac{a_k}{\left(\frac{1-t}{t}\right)^2+1/\lambda_k}.
\]
For this algebraic calculation, temporarily write
$z=((1-t)/t)^2$. The sum is then
$\sum_{k=1}^r a_k/(z+1/\lambda_k)$.
To combine these fractions, use the common denominator
$\prod_{k=1}^r(z+1/\lambda_k)$.
For example, with two terms this is the familiar identity
\[
 \frac{a_1}{z+1/\lambda_1}+\frac{a_2}{z+1/\lambda_2}
 =\frac{a_1(z+1/\lambda_2)+a_2(z+1/\lambda_1)}
 {(z+1/\lambda_1)(z+1/\lambda_2)}.
\]
For $r$ terms, the numerator of the $k$th fraction is multiplied by
all denominator factors except its own. Thus the combined numerator
is the polynomial
\[
 P(z)=\sum_{k=1}^r a_k
       \prod_{\ell\ne k}\left(z+\frac1{\lambda_\ell}\right),
 \qquad
 \sum_{k=1}^r\frac{a_k}{z+1/\lambda_k}
 =\frac{P(z)}{\displaystyle\prod_{k=1}^r(z+1/\lambda_k)}.
\]
Restoring $z=((1-t)/t)^2$ and the factor $1/t^2$, we obtain
\begin{equation}
 W(t)=
 \frac{P\!\left(\left(\frac{1-t}{t}\right)^2\right)}
 {t^2\displaystyle\prod_{k=1}^r
       \left[\left(\frac{1-t}{t}\right)^2+\frac1{\lambda_k}\right]}.
 \label{eq:calibration-rational-signs}
\end{equation}
This is an algebraic identity for the original $W(t)$ on $(0,1)$;
$P$ simply names the numerator after combining the fractions.

Each term defining $P$ contains $r-1$ linear factors, so its degree
is at most $r-1$. We must also check that these terms do not cancel
to make $P$ identically zero. As a polynomial, $P$ is defined for all
real arguments, even though the arguments used in
\eqref{eq:calibration-rational-signs} are positive.
For each $k$, evaluate it at $z=-1/\lambda_k$.
Every term other than the $k$th contains the factor
$z+1/\lambda_k$ and therefore vanishes. The remaining term gives
\[
 P(-1/\lambda_k)
 =a_k\prod_{\ell\ne k}
       \left(\frac1{\lambda_\ell}-\frac1{\lambda_k}\right).
\]
The product is nonzero because the precisions are distinct.
If $P$ were identically zero, the left-hand side would vanish for
every $k$, forcing every $a_k=0$. This contradicts
$\boldsymbol a\ne0$. Hence $P$ is a nonzero polynomial of degree
at most $r-1$.

Finally, every denominator factor in
\eqref{eq:calibration-rational-signs} is positive for $0<t<1$.
Therefore $W(t)=0$ exactly when
$P(((1-t)/t)^2)=0$, and $W(t)$ has the same sign as that numerator.
The argument $((1-t)/t)^2$ is strictly decreasing from $+\infty$ to
zero as $t$ increases from zero to one, since its derivative is
$-2(1-t)/t^3<0$.
Each positive root of $P$ can consequently correspond to only one
value of $t$; nonpositive roots correspond to none in $(0,1)$.
A nonzero polynomial of degree at most $r-1$ has at most $r-1$
distinct real roots, so $W$ has at most $r-1$ zeros in $(0,1)$.
Since $W$ is continuous, every sign change must pass through a zero.
It therefore has at most $r-1$ sign changes, and it cannot be
identically zero.

\noindent\textbf{Step 3: A polynomial with matching sign changes gives a contradiction.}
Choose a polynomial $Q$ with one simple root at each interior point
where $W$ changes sign, and with no other roots. Its degree is the
number of these points, hence at most $r-1$.
At each such point, both $Q$ and $W$ change sign, so their product
keeps the same sign. At any other zero of $W$, neither function
changes sign. Thus $QW$ has one fixed sign wherever it is nonzero.
Multiplying $Q$ by $-1$ if necessary makes $QW\geq0$.
If $W$ has no sign changes, take $Q=1$ or $Q=-1$ instead.

Since $W$ is nonzero somewhere in $(0,1)$ and $Q$ is nonzero wherever
$W$ is nonzero, the continuous product $QW$ is strictly positive on
some open interval. Hence
\[
 \int_0^1Q(t)W(t)\,dt>0.
\]
But $\deg Q\leq r-1$, so Step 1 says that this same integral is zero.
The contradiction rules out a nonzero left null vector and proves
that $L$ is invertible.
\end{proof}

\subsection{From the leading correction to an exact solution}
\label{app:calibration-uniform-expansion}

Appendix~\ref{app:calibration-family} established that $L$ is invertible.
We now use this fact to choose the schedule coefficients.
The argument has two stages: first choose $\boldsymbol K/S$ to cancel
the leading error. Then repeatedly adjust the coefficients and prove that they converge to values yielding exactly zero terminal loss at the fixed sampling budget.
We explain these two stages before deriving the error estimates they
use. 

\paragraph{What must the coefficients achieve?}
The target standard deviation in direction $k$ is
$1/\sqrt{\lambda_k}$, while the sampler produces the scale
$D_{1,k}^{\mathrm M}(\boldsymbol\theta)$.
For coefficients close to zero and sufficiently large $S$, all
midpoint factors are positive; this is verified below.
We can therefore use the logarithmic error already introduced in
Appendix~\ref{sec:condot-midpoint-asymptotics}:
\begin{equation}
 e_k(1/S;\boldsymbol\theta)
 =\log\!\left(\sqrt{\lambda_k}
 D_{1,k}^{\mathrm M}(\boldsymbol\theta)\right).
 \label{eq:calibration-log-error}
\end{equation}
The expression inside the logarithm is the generated scale divided
by the target scale. Hence
\[
 e_k(1/S;\boldsymbol\theta)=0
 \quad\Longleftrightarrow\quad
 D_{1,k}^{\mathrm M}(\boldsymbol\theta)=\frac1{\sqrt{\lambda_k}}.
\]
Thus we need to make $r$ scalar errors zero using the $r$ coefficients
in $\boldsymbol\theta$. 

\paragraph{The two estimates used in the proof.}
The first estimate below describes the error itself. The second
describes how that error changes when one coefficient changes.
We state them together, explain their consequences, and then prove
them at the end of this subsection.
Vector norms measure Euclidean length. For matrices, the norm
measures the largest factor by which the matrix can stretch a vector.

\begin{samepage}
\begin{lemma}
\label{lem:calibration-normalized-error}
Fix $R>0$. For all sufficiently large $S$ and all coefficient vectors
with $\|\boldsymbol\theta\|\leq R/S$,
\begin{align}
 e_k(1/S;\boldsymbol\theta)
 &=\frac{d_k}{S^3}
   +\frac1{S^2}\sum_{j=1}^rL_{kj}\theta_j
   +O(S^{-4}),
 \label{eq:calibration-normalized-C1}\\
 \partial_{\theta_j}e_k(1/S;\boldsymbol\theta)
 &=\frac{L_{kj}}{S^2}+O(S^{-3}).
 \label{eq:calibration-nearby-derivative}
\end{align}
The bounds on the remainders hold with the same constants for all
coefficients in this set. The constants may depend on $R$ and the
fixed target spectrum, but not on $S$.
\end{lemma}
\end{samepage}

In~\eqref{eq:calibration-normalized-C1}, the first term is CondOT's
remaining error. The second term is the leading change caused by the
schedule correction. The last term is what remains beyond these two
contributions.
Equation~\eqref{eq:calibration-nearby-derivative} says that changing
$\theta_j$ by a small amount changes the error in direction $k$
at leading rate $L_{kj}/S^2$.
This explains the relevant sizes: a coefficient change of order
$1/S$ can cancel an error of order $S^{-3}$; a further change of
order $S^{-2}$ can address the error of order $S^{-4}$ left afterward.

\paragraph{First choose the explicit leading correction.}
By Appendix~\ref{app:calibration-family}, there is a unique solution
$\boldsymbol K$ to
\begin{equation}
 L\boldsymbol K=-(d_1,\ldots,d_r)^\top.
 \label{eq:calibration-leading-coefficients}
\end{equation}
This is exactly the linear system in the main text.
Take $\boldsymbol\theta=\boldsymbol K/S$ in the first estimate.
For every $k$,
\begin{align*}
 e_k(1/S;\boldsymbol K/S)
 &=\frac{d_k}{S^3}
   +\frac1{S^2}\sum_{j=1}^rL_{kj}\frac{K_j}{S}
   +O(S^{-4})\\
 &=\frac1{S^3}
   \underbrace{\left(d_k+\sum_{j=1}^rL_{kj}K_j\right)}_{=\,0}
   +O(S^{-4})
 =O(S^{-4}).
\end{align*}
The displayed cancellation is the reason for choosing
$\boldsymbol K$. It removes the leading term in every eigendirection.

\begin{corollary}
\label{cor:calibration-first-order}
The schedule $\beta_{\boldsymbol K/S}$ satisfies
\[
J_S(\beta_{\boldsymbol K/S})=O(S^{-8})
\qquad\text{as }S\to\infty.
\]
\end{corollary}

\begin{proof}
Exponentiating~\eqref{eq:calibration-log-error} gives
\[
 D_{1,k}^{\mathrm M}(\boldsymbol K/S)-\frac1{\sqrt{\lambda_k}}
 =\frac{\exp\!\left(e_k(1/S;\boldsymbol K/S)\right)-1}
 {\sqrt{\lambda_k}}.
\]
For a small number $u$, $\exp(u)-1=u+O(u^2)$.
The logarithmic error just obtained is $O(S^{-4})$, so the scale
error on the left is also $O(S^{-4})$.
Because $\alpha_t=t$ already makes the mean exact, the Wasserstein
objective squares each scale error and sums over the
fixed number of eigendirections. Squaring gives $O(S^{-8})$, and
the finite sum has the same order, including any repeated eigenvalues.

\end{proof}

\paragraph{Next remove the remaining error exactly.}
The first correction leaves errors of size $O(S^{-4})$.
We now prove that these errors can be removed by an additional
coefficient change of size $O(S^{-2})$.
We give a rule for repeatedly adjusting the schedule coefficients and prove that the coefficients converge to values yielding exactly zero terminal loss at the fixed sampling budget \(S\).

\begin{proposition}
\label{prop:appendix-exact-calibration}
For every sufficiently large integer $S$, there are coefficients
\begin{equation}
 \boldsymbol\theta_S
 =\frac{\boldsymbol K}{S}+O(S^{-2})
 \label{eq:calibration-exact-parameters}
\end{equation}
such that midpoint sampling with noise schedule
$\beta_{\boldsymbol\theta_S}$ satisfies
\begin{equation}
 Y_S^{\mathrm M}=\mu+\Sigma^{1/2}\epsilon,
 \qquad J_S(\beta_{\boldsymbol\theta_S})=0.
 \label{eq:calibration-pathwise-exactness}
\end{equation}
The equality for $Y_S^{\mathrm M}$ holds for the same initial noise
$\epsilon$. The calibrated coefficients are locally unique within
the chosen polynomial family.
\end{proposition}

\begin{proof}
\noindent\textbf{Step 1: Correct the current error.}
Starting from a coefficient vector
$\boldsymbol\theta$, define its next value by
\[
 T_S(\boldsymbol\theta)
 =\boldsymbol\theta-S^2L^{-1}
 \begin{pmatrix}
  e_1(1/S;\boldsymbol\theta)\\[-1mm]
  \vdots\\[-1mm]
  e_r(1/S;\boldsymbol\theta)
 \end{pmatrix}.
\]
The matrix $L/S^2$ describes how coefficient changes affect the
errors, so its inverse $S^2L^{-1}$ converts the current errors into
a coefficient correction. For an exactly linear error with this
derivative, the update would remove the error in one step.
Our errors are nonlinear, so we will repeat the update.

If the update leaves the coefficients unchanged, its correction
term must be zero. Multiplying that term by $L/S^2$ recovers
the error vector, so all $r$ errors must also be zero.
Conversely, zero errors give a zero correction and leave
the coefficients unchanged.

\noindent\textbf{Step 2: Show that the update brings nearby vectors closer.}
Apply the preceding error estimates with
$R=\|\boldsymbol K\|+1$.
This choice leaves a margin of $1/S$ around the starting vector
$\boldsymbol K/S$. The further adjustments will have size
$O(S^{-2})$, which is smaller than this margin for sufficiently
large $S$. Indeed, if
$\|\boldsymbol\theta-\boldsymbol K/S\|\leq B/S^2$
for a constant $B>0$ independent of $S$, then
\[
 \|\boldsymbol\theta\|
 \leq \frac{\|\boldsymbol K\|}{S}+\frac{B}{S^2}
 \leq \frac{\|\boldsymbol K\|+1}{S}
 =\frac{R}{S}
 \qquad\text{whenever }S\geq B.
\]
Thus coefficients within $O(S^{-2})$ of $\boldsymbol K/S$
remain in the region where our error estimates apply.

The Jacobian of a map is the matrix of its first derivatives; it
describes how much the output changes when the input changes.
Write $DT_S$ for the Jacobian of this update.
By~\eqref{eq:calibration-nearby-derivative}, the Jacobian of the error
vector has the matrix form $L/S^2+O(S^{-3})$. Differentiating the
update therefore gives
\begin{align*}
 DT_S(\boldsymbol\theta)
 &=I-S^2L^{-1}\left(\frac{L}{S^2}+O(S^{-3})\right)\\
 &=I-L^{-1}L+O(S^{-1})
 =O(S^{-1}).
\end{align*}
The identity terms cancel. The remaining derivative is small, so
for sufficiently large $S$ its norm is at most $1/2$ everywhere in
the stated set.
We now explain why the derivative bound implies that the update
brings any two coefficient vectors closer together.

Take two vectors $\boldsymbol\theta$ and $\boldsymbol\eta$ satisfying
$\|\boldsymbol\theta\|,\|\boldsymbol\eta\|\leq R/S$.
To compare their updated values, move along the straight line
from $\boldsymbol\eta$ to $\boldsymbol\theta$. Its points are
\[
 (1-u)\boldsymbol\eta+u\boldsymbol\theta,
 \qquad 0\leq u\leq1.
\]
At $u=0$ we are at $\boldsymbol\eta$, and at $u=1$ we reach
$\boldsymbol\theta$. Every intermediate point remains in our region,
because the triangle inequality gives
\[
 \|(1-u)\boldsymbol\eta+u\boldsymbol\theta\|
 \leq (1-u)\|\boldsymbol\eta\|+u\|\boldsymbol\theta\|
 \leq (1-u)\frac{R}{S}+u\frac{R}{S}
 =\frac{R}{S}.
\]
Thus the bound $\|DT_S\|\leq1/2$ holds along the entire line.

By the chain rule, the rate at which the updated vector changes
along this line is
\[
 \frac{d}{du}
 T_S\big((1-u)\boldsymbol\eta+u\boldsymbol\theta\big)
 =
 DT_S\big((1-u)\boldsymbol\eta+u\boldsymbol\theta\big)
 (\boldsymbol\theta-\boldsymbol\eta).
\]
The fundamental theorem of calculus expresses the total change
as the integral of this derivative:
\[
 T_S(\boldsymbol\theta)-T_S(\boldsymbol\eta)
 =
 \int_0^1
 DT_S\big((1-u)\boldsymbol\eta+u\boldsymbol\theta\big)
 (\boldsymbol\theta-\boldsymbol\eta)\,du.
\]
Taking norms, the norm of an integral is at most the integral
of the norms. Also, $\|DT_S\|\leq1/2$ means that multiplying
any vector by this matrix produces a vector at most half as long.
Consequently,
\begin{equation}
\begin{aligned}
 \|T_S(\boldsymbol\theta)-T_S(\boldsymbol\eta)\|
 &\leq
 \int_0^1
 \left\|
 DT_S\big((1-u)\boldsymbol\eta+u\boldsymbol\theta\big)
 (\boldsymbol\theta-\boldsymbol\eta)
 \right\|\,du\\
 &\leq
 \int_0^1 \frac12
 \|\boldsymbol\theta-\boldsymbol\eta\|\,du\\
 &=\frac12\|\boldsymbol\theta-\boldsymbol\eta\|.
\end{aligned}
\label{eq:calibration-distance-halving}
\end{equation}
Therefore, after applying the update to both vectors, they are
at most half as far apart as before. A map with this property
is called a \emph{contraction}.

\noindent\textbf{Step 3: Show that the updates stay close to $\boldsymbol K/S$.}
We start at $\boldsymbol K/S$, whose errors are $O(S^{-4})$.
Multiplication by the fixed matrix $L^{-1}$ and by $S^2$ shows that
the first coefficient update has size $O(S^{-2})$.
Thus there is a constant $C>0$, independent of $S$, such that
\[
 \|T_S(\boldsymbol K/S)-\boldsymbol K/S\|
 \leq\frac{C}{S^2}.
\]
Consider all coefficient vectors within distance $2C/S^2$ of
$\boldsymbol K/S$. For sufficiently large $S$, these vectors lie
in the set from Step 2, since
\[
 \|\boldsymbol\theta\|
 \leq\frac{\|\boldsymbol K\|}{S}+\frac{2C}{S^2}
 \leq\frac{\|\boldsymbol K\|+1}{S}=\frac{R}{S}.
\]
For any such vector, the triangle inequality and
\eqref{eq:calibration-distance-halving} imply
\begin{align*}
 \|T_S(\boldsymbol\theta)-\boldsymbol K/S\|
 &\leq
 \|T_S(\boldsymbol\theta)-T_S(\boldsymbol K/S)\|
 +\|T_S(\boldsymbol K/S)-\boldsymbol K/S\|\\
 &\leq\frac12\|\boldsymbol\theta-\boldsymbol K/S\|
       +\frac{C}{S^2}
 \leq\frac12\frac{2C}{S^2}+\frac{C}{S^2}
 =\frac{2C}{S^2}.
\end{align*}
So an update cannot leave this smaller closed ball.
Starting at its center and applying the update repeatedly keeps
every subsequent vector in the region where Step 2 applies.

\noindent\textbf{Step 4: Show that the updates converge to an exact solution.}
The first update changes the coefficients by at most $C/S^2$,
as shown in Step 3.

To bound the second change, apply the contraction inequality
to the starting vector and the vector after the first update.
Their outputs are the vectors after the first and second updates.
Thus the distance between these outputs (the size of the second
change) is at most half the size of the first change.

Repeating this argument shows that each change is at most half
as large as the preceding one. Therefore:
\[
 \text{first change}\leq\frac{C}{S^2},\qquad
 \text{second change}\leq\frac{C}{2S^2},\qquad
 \text{third change}\leq\frac{C}{4S^2},
 \quad\ldots.
\]
Adding the lengths of all these changes gives
\[
 \frac{C}{S^2}
 +\frac{C}{2S^2}
 +\frac{C}{4S^2}
 +\cdots
 =
 \frac{C}{S^2}
 \left(1+\frac12+\frac14+\cdots\right)
 =\frac{2C}{S^2}.
\]
Hence the total distance traveled by the coefficient vectors,
even over infinitely many updates, is at most $2C/S^2$. The sum of the remaining changes after
the first $m$ updates is at most $2C/(2^mS^2)$, which tends to zero
as $m\to\infty$.
Once sufficiently many updates have been made, all later
vectors are arbitrarily close to one another.
Thus the coefficient vectors form a Cauchy sequence: any two
sufficiently late vectors are arbitrarily close. Since every
Cauchy sequence in $\mathbb R^r$ converges, the updates approach
a limiting coefficient vector, which we denote by
$\boldsymbol\theta_S$.

The ball is closed, so the limit remains in it. The map $T_S$ is
continuous, so taking limits in the update rule gives
$T_S(\boldsymbol\theta_S)=\boldsymbol\theta_S$.
By Step 1, all logarithmic errors at this limit are exactly zero.
Also,
\[
 \|\boldsymbol\theta_S-\boldsymbol K/S\|
 \leq\frac{2C}{S^2},
\]
which proves~\eqref{eq:calibration-exact-parameters}.
This proves existence of an exact root for the fixed finite budget.

For uniqueness, suppose two vectors in the region of Step 2 both
solve the exact equations. Both are fixed points, so
\eqref{eq:calibration-distance-halving} says that their distance is
at most half of itself. That distance must be zero.
This gives the asserted local uniqueness.
Every estimate above holds once $S$ exceeds a fixed threshold, so
the argument applies separately to every integer beyond that threshold.

\noindent\textbf{Step 5: Translate zero logarithmic error into exact sampling.}
For sufficiently large $S$, the midpoint factors at
$\boldsymbol\theta_S$ are positive, and
\[
 e_k(1/S;\boldsymbol\theta_S)=0
 \quad\Longrightarrow\quad
 D_{1,k}^{\mathrm M}(\boldsymbol\theta_S)
 =\frac1{\sqrt{\lambda_k}}.
\]
Repeated eigenvalues have the same multiplier, so this gives the
correct scale in every eigendirection.
The mean is already exact because $\alpha_t=t$
(Proposition~\ref{prop:linear-alpha-unique}).
Consequently, in the covariance eigenbasis,
\[
 Y_{S,k}^{U,\mathrm M}
 =\mu_k^U+\frac{\epsilon_k^U}{\sqrt{\lambda_k}},
\]
with eigenvalues repeated according to their multiplicities.
Transforming back gives $Y_S^{\mathrm M}=\mu+\Sigma^{1/2}\epsilon$.
Every term in the exact squared Wasserstein objective is therefore zero.
\end{proof}

\paragraph{Derivation of the two error estimates.}
It remains to prove the two estimates in
Lemma~\ref{lem:calibration-normalized-error}.
We first prove the coefficient derivative estimate
\eqref{eq:calibration-nearby-derivative}, then integrate that
estimate to obtain the error formula
\eqref{eq:calibration-normalized-C1}.
Step 1 provides bounds that apply uniformly to the schedules
under consideration. Step 2 derives the derivative expansion from
the exact midpoint factors. Step 3 identifies its leading coefficient
as $L_{kj}$, and Step 4 combines it with the known CondOT error.

\begin{proof}[Proof of Lemma~\ref{lem:calibration-normalized-error}]
\noindent\textbf{Step 1: Keep all denominators and logarithms well defined.}
Work on the fixed closed coefficient ball
$\|\boldsymbol\theta\|\leq1$.
For sufficiently large $S$, this ball contains all coefficients
satisfying $\|\boldsymbol\theta\|\leq R/S$.

Recall the variance and scalar field:
\[
 v_k(t;\boldsymbol\theta)
 =\beta_{\boldsymbol\theta}(t)^2+\frac{t^2}{\lambda_k},
 \qquad
 A_k(t;\boldsymbol\theta)
 =\frac{\partial_t v_k(t;\boldsymbol\theta)}
 {2v_k(t;\boldsymbol\theta)}.
\]
For every coefficient vector,
$v_k(0;\boldsymbol\theta)=1$, while for $t>0$,
\[
 v_k(t;\boldsymbol\theta)
 \geq\frac{t^2}{\lambda_k}>0.
\]
The variance is continuous in both $t$ and $\boldsymbol\theta$.
Since $[0,1]$ and the closed coefficient ball form a compact set,
the variance attains a strictly positive minimum on this set.
Thus there exists $c_k>0$ such that
\[
 v_k(t;\boldsymbol\theta)\geq c_k
 \qquad\text{for all }t\in[0,1]
 \text{ and }\|\boldsymbol\theta\|\leq1.
\]
This lower bound is independent of $t$, $\boldsymbol\theta$,
and $S$ within the stated region.

The variance $v_k$ is a polynomial in $t$ and the coefficients,
and $A_k$ is a ratio of such polynomials. Differentiating $A_k$
with respect to time, the coefficients, or both produces
polynomial expressions divided by powers of $v_k$.
Since $v_k$ never vanishes in this region, these derivatives
are continuous there. Time and the coefficients range over
a closed, bounded set, so each derivative needed below has
a finite maximum in absolute value. Thus each derivative has
a bound that works for all times and all coefficients in the ball.

The same argument bounds $A_k$ itself. The midpoint factor is
\[
 g_k^{\mathrm M}(n;\boldsymbol\theta)
 =1+\frac{A_k(t_{n+\frac12};\boldsymbol\theta)}{S}
   +\frac{A_k(t_{n+\frac12};\boldsymbol\theta)
           A_k(t_n;\boldsymbol\theta)}{2S^2}.
\]
Both evaluations of \(A_k\), at the start and midpoint of the step, are bounded in absolute value by the same constant, for every step and every coefficient vector in the chosen ball.
The two terms added to $1$ are therefore $O(S^{-1})$ and
$O(S^{-2})$, with constants independent of $n$ and
$\boldsymbol\theta$ in the ball.
Consequently, for sufficiently large $S$,
\[
 \big|g_k^{\mathrm M}(n;\boldsymbol\theta)-1\big|
 \leq\frac12,
 \qquad\text{so}\qquad
 \frac12\leq g_k^{\mathrm M}(n;\boldsymbol\theta)\leq\frac32,
\]
for every step and every coefficient vector under consideration.
All midpoint factors are therefore positive. Their product is
positive as well, so the logarithmic error is well defined.

\noindent\textbf{Step 2: Derive the coefficient derivative of the error.}
Equation~\eqref{eq:uniform-global-error-sum} gives, with $h=1/S$,
\begin{equation}
 e_k(1/S;\boldsymbol\theta)
 =-\frac1{48S^2}
 \left[
 \frac1S\sum_{n=0}^{S-1}
 \frac{\partial_t^3v_k(t_n;\boldsymbol\theta)}
      {v_k(t_n;\boldsymbol\theta)}
 \right]+O(S^{-3}).
 \label{eq:calibration-differentiated-expansion}
\end{equation}
We need a corresponding formula for
$\partial_{\theta_j}e_k(1/S;\boldsymbol\theta)$.
The displayed sum can be differentiated term by term. However, a bound
on the remainder itself does not tell us how quickly it changes with
$\theta_j$. We therefore calculate the derivative from the exact
midpoint factors and bound the omitted terms in that calculation.

\medskip
\noindent\emph{(a) Differentiate the exact expression.}
Equation~\eqref{eq:global-error-local-defects} writes the total
logarithmic error as a sum of one-step errors.
Differentiating this exact identity gives
\[
 \partial_{\theta_j}e_k(1/S;\boldsymbol\theta)
 =\sum_{n=0}^{S-1}\left[
 \frac{\partial_{\theta_j}g_k^{\mathrm M}(n;\boldsymbol\theta)}
      {g_k^{\mathrm M}(n;\boldsymbol\theta)}
 -\int_{t_n}^{t_{n+1}}
       \partial_{\theta_j}A_k(t;\boldsymbol\theta)\,dt
 \right].
\]
Here we used the derivative of a logarithm. Differentiation with respect to \(\theta_j\) can pass inside the integral because both \(A_k\) and its derivative with respect to \(\theta_j\) are continuous on the region established in Step 1. We will expand the two terms
inside the brackets and then subtract them.

Throughout this step, primes mean time derivatives, and
$\partial_{\theta_j}$ means a coefficient derivative.
These derivatives can be taken in either order because \(A_k\)
is smooth on the region established in Step 1.
When arguments are omitted in the next displays, the fields and
their derivatives are evaluated at $(t_n;\boldsymbol\theta)$.

\medskip
\noindent\emph{(b) Expand the derivative of the midpoint factor.}
Taylor's theorem, applied in time from $t_n$ to $t_n+1/(2S)$,
gives the following two expansions:
\begin{align*}
 A_k(t_{n+\frac12};\boldsymbol\theta)
 &=A_k+\frac{A_k'}{2S}+\frac{A_k''}{8S^2}+O(S^{-3}),\\
 \partial_{\theta_j}A_k(t_{n+\frac12};\boldsymbol\theta)
 &=\partial_{\theta_j}A_k
   +\frac{\partial_{\theta_j}A_k'}{2S}
   +\frac{\partial_{\theta_j}A_k''}{8S^2}+O(S^{-3}).
\end{align*}
The second line is a Taylor expansion of the function
$t\mapsto\partial_{\theta_j}A_k(t;\boldsymbol\theta)$ itself.
For each line, Taylor's theorem bounds the remainder by a constant
times $(1/(2S))^3$, using the third time derivative of that function.
Step 1 bounds both $\partial_t^3A_k$ and
$\partial_t^3\partial_{\theta_j}A_k$ uniformly, so both remainders
are $O(S^{-3})$ with constants independent of $n$ and the coefficients.

Differentiate the midpoint factor from Step 1 using the product rule:
\begin{align*}
 \partial_{\theta_j}g_k^{\mathrm M}(n;\boldsymbol\theta)
 &=\frac{\partial_{\theta_j}A_k(t_{n+\frac12};\boldsymbol\theta)}{S}+\frac{1}{2S^2}\Big[
   \partial_{\theta_j}A_k(t_{n+\frac12};\boldsymbol\theta)
       A_k(t_n;\boldsymbol\theta)
   +A_k(t_{n+\frac12};\boldsymbol\theta)
       \partial_{\theta_j}A_k(t_n;\boldsymbol\theta)\Big].
\end{align*}
Substitute the two Taylor expansions above and collect equal powers
of $1/S$. This gives
\begin{align*}
 \partial_{\theta_j}g_k^{\mathrm M}(n;\boldsymbol\theta)
 &=\frac{\partial_{\theta_j}A_k}{S}
 +\frac{\partial_{\theta_j}A_k'
        +2A_k\partial_{\theta_j}A_k}{2S^2}+\frac1{S^3}\left[
   \frac{\partial_{\theta_j}A_k''}{8}
   +\frac{A_k'\partial_{\theta_j}A_k
          +A_k\partial_{\theta_j}A_k'}{4}
   \right]+O(S^{-4}).
\end{align*}

We also need the reciprocal of the midpoint factor. Substituting the first Taylor expansion of $A_k$ above into
the definition of $g_k^{\mathrm M}$ gives
\[
 g_k^{\mathrm M}(n;\boldsymbol\theta)
 =1+\frac{A_k}{S}+\frac{A_k'+A_k^2}{2S^2}+O(S^{-3}).
\]
For the factor on this step, use the exact algebraic identity
\[
 \frac1{g_k^{\mathrm M}}
 =1-(g_k^{\mathrm M}-1)+(g_k^{\mathrm M}-1)^2
   -\frac{(g_k^{\mathrm M}-1)^3}{g_k^{\mathrm M}}.
\]
The last term is $O(S^{-3})$, because
$g_k^{\mathrm M}-1=O(S^{-1})$ and $1/g_k^{\mathrm M}\leq2$.
Substituting the factor expansion into this identity gives
\[
 \frac1{g_k^{\mathrm M}(n;\boldsymbol\theta)}
 =1-\frac{A_k}{S}+\frac{A_k^2-A_k'}{2S^2}+O(S^{-3}).
\]
The coefficient of $S^{-2}$ here is
$A_k^2-(A_k'+A_k^2)/2=(A_k^2-A_k')/2$.
All bounds used here are uniform over the steps and coefficients.

Multiply the expansions of $\partial_{\theta_j}g_k^{\mathrm M}$
and $1/g_k^{\mathrm M}$. The terms of order $S^{-2}$ containing
$A_k\partial_{\theta_j}A_k$ cancel. The result is
\begin{align*}
 \frac{\partial_{\theta_j}g_k^{\mathrm M}(n;\boldsymbol\theta)}
      {g_k^{\mathrm M}(n;\boldsymbol\theta)}
 &=\frac{\partial_{\theta_j}A_k}{S}
   +\frac{\partial_{\theta_j}A_k'}{2S^2}+\frac1{S^3}\left[
   \frac{\partial_{\theta_j}A_k''}{8}
   -\frac{A_k'\partial_{\theta_j}A_k
          +A_k\partial_{\theta_j}A_k'}{4}
   -\frac{A_k^2\partial_{\theta_j}A_k}{2}
   \right]+O(S^{-4}).
\end{align*}
The remainder stays $O(S^{-4})$: the $O(S^{-4})$ remainder of
$\partial_{\theta_j}g_k^{\mathrm M}$ multiplies a bounded reciprocal,
and the reciprocal's $O(S^{-3})$ remainder multiplies
$\partial_{\theta_j}g_k^{\mathrm M}=O(S^{-1})$.
Products of displayed terms whose powers add to at least four
also contribute only $O(S^{-4})$.

\medskip
\noindent\emph{(c) Subtract the integral term.}
Taylor expanding $\partial_{\theta_j}A_k(t;\boldsymbol\theta)$
about $t_n$ and integrating over an interval of length $1/S$ gives
\[
 \int_{t_n}^{t_{n+1}}
 \partial_{\theta_j}A_k(t;\boldsymbol\theta)\,dt
 =\frac{\partial_{\theta_j}A_k}{S}
  +\frac{\partial_{\theta_j}A_k'}{2S^2}
  +\frac{\partial_{\theta_j}A_k''}{6S^3}
  +O(S^{-4}).
\]

Now subtract this formula from the preceding expansion of
$\partial_{\theta_j}g_k^{\mathrm M}/g_k^{\mathrm M}$.
The terms of orders $S^{-1}$ and $S^{-2}$ are identical and cancel.
At order $S^{-3}$, the coefficient of
$\partial_{\theta_j}A_k''$ becomes $1/8-1/6=-1/24$.
Thus the derivative of the one-step error is
\begin{align*}
 &\frac{\partial_{\theta_j}g_k^{\mathrm M}(n;\boldsymbol\theta)}
       {g_k^{\mathrm M}(n;\boldsymbol\theta)}
 -\int_{t_n}^{t_{n+1}}
       \partial_{\theta_j}A_k(t;\boldsymbol\theta)\,dt\\
 &\quad=-\frac1{24S^3}\left[
   \partial_{\theta_j}A_k''
   +6A_k'\partial_{\theta_j}A_k
   +6A_k\partial_{\theta_j}A_k'
   +12A_k^2\partial_{\theta_j}A_k
   \right]+O(S^{-4})\\
 &\quad=-\frac1{24S^3}
   \partial_{\theta_j}\big(A_k''+6A_kA_k'+4A_k^3\big)
   +O(S^{-4}).
\end{align*}
The last equality uses the product rule on $A_kA_k'$ and the
power rule on $A_k^3$.

To write this in terms of $v_k$, use the exact identity established
in the proof of Proposition~\ref{prop:local-log-defect}:
\[
 \frac{\partial_t^3v_k}{v_k}
 =2\big(A_k''+6A_kA_k'+4A_k^3\big).
\]
Differentiating this exact identity with respect to $\theta_j$
gives the one-step formula
\begin{align*}
 &\frac{\partial_{\theta_j}g_k^{\mathrm M}(n;\boldsymbol\theta)}
       {g_k^{\mathrm M}(n;\boldsymbol\theta)}
 -\int_{t_n}^{t_{n+1}}
       \partial_{\theta_j}A_k(t;\boldsymbol\theta)\,dt\\
 &\qquad=-\frac1{48S^3}
 \partial_{\theta_j}\left(
 \frac{\partial_t^3v_k(t_n;\boldsymbol\theta)}
      {v_k(t_n;\boldsymbol\theta)}\right)+O(S^{-4}).
\end{align*}

\medskip
\noindent\emph{(d) Sum the steps and replace the sum by an integral.}
The remainder bound above is the same for every step. Consequently,
the sum of the $S$ remainders is $S\,O(S^{-4})=O(S^{-3})$.
Substitution into the exact sum in part (a) gives
\[
 \partial_{\theta_j}e_k(1/S;\boldsymbol\theta)
 =-\frac1{48S^2}
 \left[
 \frac1S\sum_{n=0}^{S-1}
 \partial_{\theta_j}\left(
 \frac{\partial_t^3v_k(t_n;\boldsymbol\theta)}
      {v_k(t_n;\boldsymbol\theta)}
 \right)
 \right]+O(S^{-3}).
\]
The bracket is the left-endpoint Riemann sum for the integral below.
Its integrand has a bounded time derivative by Step 1.
On an interval of length $1/S$, the mean value theorem therefore
bounds the difference between the integrand and its left-endpoint
value by a constant times $1/S$.
Integrating this difference over one interval gives $O(S^{-2})$;
summing over all $S$ intervals gives $O(S^{-1})$.
Finally, the bracket is multiplied by $1/(48S^2)$, so replacing
the sum by the integral changes the formula by $O(S^{-3})$.
We obtain
\begin{equation}
 \partial_{\theta_j}e_k(1/S;\boldsymbol\theta)
 =-\frac1{48S^2}\int_0^1
 \partial_{\theta_j}
 \left(\frac{\partial_t^3v_k(t;\boldsymbol\theta)}
 {v_k(t;\boldsymbol\theta)}\right)\,dt+O(S^{-3}).
 \label{eq:calibration-parameter-derivatives}
\end{equation}
Every remainder bound in this step uses constants independent
of $S$, the step index, and the coefficients in the fixed ball
chosen in Step 1.

\noindent\textbf{Step 3: Evaluate the leading derivative at CondOT and control its change nearby.}
We first evaluate the integral in
\eqref{eq:calibration-parameter-derivatives} at
$\boldsymbol\theta=0$.
From the polynomial family~\eqref{eq:calibration-family},
\[
 \partial_{\theta_j}\beta_{\boldsymbol\theta}(t)=t^j(1-t).
\]
Since $v_k=\beta_{\boldsymbol\theta}^2+t^2/\lambda_k$,
the chain rule gives
\[
 \partial_{\theta_j}v_k(t;\boldsymbol\theta)
 =2\beta_{\boldsymbol\theta}(t)t^j(1-t).
\]
At $\boldsymbol\theta=0$, the schedule is $1-t$, so
\[
 \partial_{\theta_j}v_k(t;0)=2t^j(1-t)^2,
 \qquad v_k(t;0)=(1-t)^2+\frac{t^2}{\lambda_k}.
\]
This last variance is quadratic in $t$, hence its third time
derivative is zero.

Apply the quotient rule to the differentiated integrand:
\[
 \partial_{\theta_j}\left(\frac{\partial_t^3v_k}{v_k}\right)
 =\frac{\partial_t^3\partial_{\theta_j}v_k}{v_k}
  -\frac{(\partial_t^3v_k)(\partial_{\theta_j}v_k)}{v_k^2}.
\]
The time and coefficient derivatives can be taken in either order
because $v_k$ is a polynomial in these variables.
At $\boldsymbol\theta=0$, the second term vanishes, and substitution
of the formulas above gives
\[
 \left.
 \partial_{\theta_j}\left(\frac{\partial_t^3v_k}{v_k}\right)
 \right|_{\boldsymbol\theta=0}
 =\frac{\big[2t^j(1-t)^2\big]'''}
 {(1-t)^2+t^2/\lambda_k}.
\]
Multiplying by $-1/48$ and integrating over $[0,1]$ gives
exactly $L_{kj}$, as defined in~\eqref{eq:calibration-response}.

We next compare the integral at $\boldsymbol\theta$ with its value
at CondOT, where $\boldsymbol\theta=0$. We first bound the change
of the integrand at each fixed time $t$, and then integrate this
bound over time.

Fix $t\in[0,1]$. To move from $0$ to $\boldsymbol\theta$,
consider the coefficient vector $u\boldsymbol\theta$ as $u$
increases from $0$ to $1$. This path stays in the ball from Step 1,
because
\[
 \|u\boldsymbol\theta\|
 =u\|\boldsymbol\theta\|
 \leq\|\boldsymbol\theta\|.
\]
Here $u$ describes movement in the coefficients; the time $t$
remains fixed.

Along this path, the $\ell$th coefficient is $u\theta_\ell$,
whose derivative with respect to $u$ is $\theta_\ell$.
The chain rule therefore gives
\[
 \frac{d}{du}\left[
 \partial_{\theta_j}\left(\frac{\partial_t^3v_k}{v_k}\right)
       (t;u\boldsymbol\theta)
 \right]
 =
 \sum_{\ell=1}^r\theta_\ell
 \partial_{\theta_\ell}\partial_{\theta_j}
       \left(\frac{\partial_t^3v_k}{v_k}\right)
       (t;u\boldsymbol\theta).
\]
The sum accounts for the fact that all coefficients change along
the path. Each derivative on the right is taken with respect to
the coefficients and then evaluated at $u\boldsymbol\theta$.

Integrating this derivative from $u=0$ to $u=1$ gives the value
at the end of the path minus the value at its start:
\begin{align*}
 &\partial_{\theta_j}\left(\frac{\partial_t^3v_k}{v_k}\right)
       (t;\boldsymbol\theta)
 -\partial_{\theta_j}\left(\frac{\partial_t^3v_k}{v_k}\right)(t;0)\\
 &\quad=\int_0^1\sum_{\ell=1}^r\theta_\ell
 \partial_{\theta_\ell}\partial_{\theta_j}
       \left(\frac{\partial_t^3v_k}{v_k}\right)
       (t;u\boldsymbol\theta)\,du.
\end{align*}
This is the fundamental theorem of calculus applied along the path.

We now bound the right-hand side. Its second coefficient derivatives
are polynomial expressions divided by powers of $v_k$.
The bounds established in Step 1 therefore bound their absolute
values by a common constant, independent of $t$, $u$, and the
coefficients in the ball.
Taking absolute values inside the integral gives
\begin{align*}
 &\left|
 \partial_{\theta_j}\left(\frac{\partial_t^3v_k}{v_k}\right)
       (t;\boldsymbol\theta)
 -\partial_{\theta_j}\left(\frac{\partial_t^3v_k}{v_k}\right)(t;0)
 \right|\\
 &\quad\leq
 \int_0^1\sum_{\ell=1}^r|\theta_\ell|
 \left|
 \partial_{\theta_\ell}\partial_{\theta_j}
       \left(\frac{\partial_t^3v_k}{v_k}\right)
       (t;u\boldsymbol\theta)
 \right|du\\
 &\quad=O\!\left(\sum_{\ell=1}^r|\theta_\ell|\right).
\end{align*}
The last line follows because each derivative is bounded and
the integration interval in $u$ has length one.

By the Cauchy--Schwarz inequality,
\[
 \sum_{\ell=1}^r|\theta_\ell|
 \leq
 \left(\sum_{\ell=1}^r1^2\right)^{1/2}
 \left(\sum_{\ell=1}^r\theta_\ell^2\right)^{1/2}
 =\sqrt r\,\|\boldsymbol\theta\|.
\]
Since $r$ is fixed, the change in the integrand is therefore
$O(\|\boldsymbol\theta\|)$, with the same bound for every $t$.

We now integrate over time. The interval $0\leq t\leq1$ also
has length one, so integrating the difference preserves the
$O(\|\boldsymbol\theta\|)$ bound. At $\boldsymbol\theta=0$,
we already showed that the integral, multiplied by $-1/48$,
equals $L_{kj}$. Hence
\[
 -\frac1{48}\int_0^1
 \partial_{\theta_j}\left(\frac{\partial_t^3v_k}{v_k}\right)
       (t;\boldsymbol\theta)\,dt
 =L_{kj}+O(\|\boldsymbol\theta\|).
\]

Substituting this into
\eqref{eq:calibration-parameter-derivatives} gives
\begin{align*}
 \partial_{\theta_j}e_k(1/S;\boldsymbol\theta)
 &=\frac1{S^2}
   \left[L_{kj}+O(\|\boldsymbol\theta\|)\right]
   +O(S^{-3})\\
 &=\frac{L_{kj}}{S^2}
   +O\!\left(\frac{\|\boldsymbol\theta\|}{S^2}\right)
   +O(S^{-3}).
\end{align*}
Finally, the lemma restricts the coefficients to
$\|\boldsymbol\theta\|\leq R/S$. Thus
\[
 \frac{\|\boldsymbol\theta\|}{S^2}\leq\frac{R}{S^3}.
\]
Because $R$ is fixed, the two remainder terms are both
$O(S^{-3})$. We conclude that
\[
 \partial_{\theta_j}e_k(1/S;\boldsymbol\theta)
 =\frac{L_{kj}}{S^2}+O(S^{-3}),
\]
which proves~\eqref{eq:calibration-nearby-derivative}.

\noindent\textbf{Step 4: Recover the error by integrating its coefficient derivative.}
Step 3 describes how the error changes as the coefficients move.
To recover the error at $\boldsymbol\theta$, start from its known
value at CondOT and again follow the path $u\boldsymbol\theta$.
The fundamental theorem of calculus and the chain rule give
\begin{align*}
 e_k(1/S;\boldsymbol\theta)-e_k(1/S;0)
 &=\int_0^1\frac{d}{du}e_k(1/S;u\boldsymbol\theta)\,du\\
 &=\int_0^1\sum_{j=1}^r\theta_j
      \partial_{\theta_j}e_k(1/S;u\boldsymbol\theta)\,du.
\end{align*}
The factor $\theta_j$ appears because the $j$th coefficient on
this path is $u\theta_j$, whose derivative with respect to $u$
is $\theta_j$.
Moreover,
$\|u\boldsymbol\theta\|=u\|\boldsymbol\theta\|\leq R/S$,
so the estimate from Step 3 holds at every point of the path,
with the same remainder bound.

Substitute that estimate into the integral:
\begin{align*}
 e_k(1/S;\boldsymbol\theta)-e_k(1/S;0)
 &=\int_0^1\sum_{j=1}^r\theta_j
       \left[\frac{L_{kj}}{S^2}+O(S^{-3})\right]du\\
 &=\frac1{S^2}\sum_{j=1}^rL_{kj}\theta_j
   +O\!\left(S^{-3}\sum_{j=1}^r|\theta_j|\right).
\end{align*}
For the remainder, we sum absolute values because the coefficients
can have either sign. The bound from Step 3 is uniform in $u$,
and the integration interval has length one.
Using Cauchy--Schwarz again,
\[
 S^{-3}\sum_{j=1}^r|\theta_j|
 \leq\sqrt r\,S^{-3}\|\boldsymbol\theta\|
 \leq\frac{\sqrt r\,R}{S^4}.
\]
The remainder is therefore $O(S^{-4})$.

Finally, at $\boldsymbol\theta=0$ we use the sharper CondOT
expansion~\eqref{eq:condot-log-error}, with $h=1/S$:
\[
 e_k(1/S;0)=\frac{d_k}{S^3}+O(S^{-4}).
\]
Adding this known error to the change just computed yields
\[
 e_k(1/S;\boldsymbol\theta)
 =\frac{d_k}{S^3}
  +\frac1{S^2}\sum_{j=1}^rL_{kj}\theta_j
  +O(S^{-4}),
\]
which is~\eqref{eq:calibration-normalized-C1}.
Both estimates in the lemma have now been proved.
\end{proof}

\newpage

\subsection{Numerical calibration procedure}
\label{app:calibration-numerics}

The high-precision computations use the four targets in
Table~\ref{tab:targets}, $\alpha_t=t$, uniform grids, and 90-digit decimal
arithmetic.
For each target, we compute $L$ once by numerically evaluating
the integrals in~\eqref{eq:calibration-response}, and then
solve~\eqref{m:calibration-linear-system} once for $\boldsymbol K$.
At each budget, we use the iteration from the proof of
Proposition~\ref{prop:appendix-exact-calibration} in
Appendix~\ref{app:calibration-uniform-expansion}:
\begin{align*}
 \boldsymbol\theta^{(0)}&=\frac{\boldsymbol K}{S},\\
 \boldsymbol\theta^{(m+1)}
 &=\boldsymbol\theta^{(m)}-S^2L^{-1}
 \begin{pmatrix}
 e_1(1/S;\boldsymbol\theta^{(m)})\\[-2pt]
 \vdots\\[-2pt]
 e_r(1/S;\boldsymbol\theta^{(m)})
 \end{pmatrix},
\end{align*}
where $e_k$ is the logarithmic scale error
in~\eqref{eq:calibration-log-error}.
Each iteration applies the full update. We stop successfully when all
midpoint factors are positive and
$\max_k|D_{1,k}^{\mathrm M}-1/\sqrt{\lambda_k}|<10^{-86}$.
Otherwise, we stop after 600 updates, or earlier if a midpoint factor
becomes nonpositive.

Table~\ref{tab:calibration-all} compares the finite-budget losses.
The first-order correction is asymptotic: it can increase the loss at small
budgets, as seen for Target D at $S=64,128$ and, slightly, for Target C at
$S=128$. For all four targets, the first-order correction improves on CondOT
at every tested budget $S\ge256$.
The iteration reaches its tolerance for all targets at
$S=256,512,1024$.
For Target C at $S=64,128$, the iteration reaches the 600-update cap
without satisfying the tolerance. For Target D at these two budgets,
a midpoint factor becomes nonpositive after four and eight updates,
respectively. These outcomes do not show that an exact schedule does not exist.
Theorem~\ref{m:exact-calibration} only guarantees calibration for sufficiently
large $S$, with a threshold depending on the target spectrum.

\begin{table}[H]
\centering
\caption{High-precision proof iteration across targets.
$J_S^{\mathrm{CondOT}}$ and $J_S^{K/S}$ denote the baseline and first-order
losses. The ``Final $J_S$'' column reports the loss at the last iterate.
Successful runs achieve $J_S<10^{-170}$.
An asterisk marks runs that reached the 600-update limit without
satisfying the stopping criterion.
A dash marks a run stopped after a nonpositive midpoint factor;
``updates'' counts applications of the proof iteration.}
\label{tab:calibration-all}
\begin{tabular}{crrrrr}
\toprule
Target & $S$ & $J_S^{\mathrm{CondOT}}$ & $J_S^{K/S}$ & Final $J_S$ & Updates\\
\midrule
A & $64$ & $2.07\times10^{-12}$ & $3.24\times10^{-16}$ & $<10^{-170}$ & 47\\
A & $128$ & $3.23\times10^{-14}$ & $1.25\times10^{-18}$ & $<10^{-170}$ & 39\\
A & $256$ & $5.05\times10^{-16}$ & $4.87\times10^{-21}$ & $<10^{-170}$ & 34\\
A & $512$ & $7.89\times10^{-18}$ & $1.90\times10^{-23}$ & $<10^{-170}$ & 29\\
A & $1024$ & $1.23\times10^{-19}$ & $7.40\times10^{-26}$ & $<10^{-170}$ & 26\\
\addlinespace[3pt]
B & $64$ & $1.91\times10^{-12}$ & $1.34\times10^{-14}$ & $<10^{-170}$ & 72\\
B & $128$ & $2.99\times10^{-14}$ & $6.03\times10^{-17}$ & $<10^{-170}$ & 56\\
B & $256$ & $4.67\times10^{-16}$ & $2.51\times10^{-19}$ & $<10^{-170}$ & 45\\
B & $512$ & $7.30\times10^{-18}$ & $1.01\times10^{-21}$ & $<10^{-170}$ & 38\\
B & $1024$ & $1.14\times10^{-19}$ & $4.02\times10^{-24}$ & $<10^{-170}$ & 33\\
\addlinespace[3pt]
C & $64$ & $1.97\times10^{-7}$ & $5.39\times10^{-8}$ & $1.07\times10^{-7}\,{}^{*}$ & 600\\
C & $128$ & $2.39\times10^{-9}$ & $2.40\times10^{-9}$ & $1.35\times10^{-9}\,{}^{*}$ & 600\\
C & $256$ & $3.93\times10^{-11}$ & $2.29\times10^{-11}$ & $<10^{-170}$ & 536\\
C & $512$ & $6.25\times10^{-13}$ & $1.35\times10^{-13}$ & $<10^{-170}$ & 210\\
C & $1024$ & $9.80\times10^{-15}$ & $6.43\times10^{-16}$ & $<10^{-170}$ & 125\\
\addlinespace[3pt]
D & $64$ & $8.21\times10^{-9}$ & $6.91\times10^{-8}$ & -- & 4\\
D & $128$ & $1.34\times10^{-10}$ & $2.61\times10^{-10}$ & -- & 8\\
D & $256$ & $2.14\times10^{-12}$ & $1.01\times10^{-12}$ & $<10^{-170}$ & 570\\
D & $512$ & $3.37\times10^{-14}$ & $3.94\times10^{-15}$ & $<10^{-170}$ & 162\\
D & $1024$ & $5.28\times10^{-16}$ & $1.54\times10^{-17}$ & $<10^{-170}$ & 100\\
\bottomrule
\end{tabular}
\end{table}

\subsection{Convergence and calibration across the four targets}
\label{app:four-target-results}

Table~\ref{tab:slopes} reports the final observed decay exponent
$\log_2(J_{32768}/J_{65536})$ for each curve and target. The values approach the predicted exponents: four for the fixed
alternatives, six for CondOT, and eight for the first-order correction in
these examples.

\begin{table}[H]
\centering
\caption{Observed decay exponents}
\label{tab:slopes}
\begin{tabular}{crrrrrr}
\toprule
Target & $(1-t)^2$ & $1-t^2$ & $\cos(\pi t/2)$ & $(1-t)(1+0.1t)$ & CondOT & $\boldsymbol K/S$\\
\midrule
A & $4.0000$ & $4.0000$ & $4.0000$ & $4.0001$ & $6.0000$ & $8.0000$\\
B & $4.0000$ & $4.0000$ & $4.0000$ & $4.0001$ & $6.0000$ & $7.9997$\\
C & $4.0000$ & $4.0370$ & $4.0043$ & $4.0478$ & $6.0000$ & $7.9959$\\
D & $4.0000$ & $4.0057$ & $4.0011$ & $4.0077$ & $6.0000$ & $8.0000$\\
\bottomrule
\end{tabular}
\end{table}

Figure~\ref{fig:all-fixed} repeats the fixed-schedule comparison in
Figure~\ref{fig:accuracy} for every target.

\begin{figure}[H]
\centering
\includegraphics[width=\linewidth]{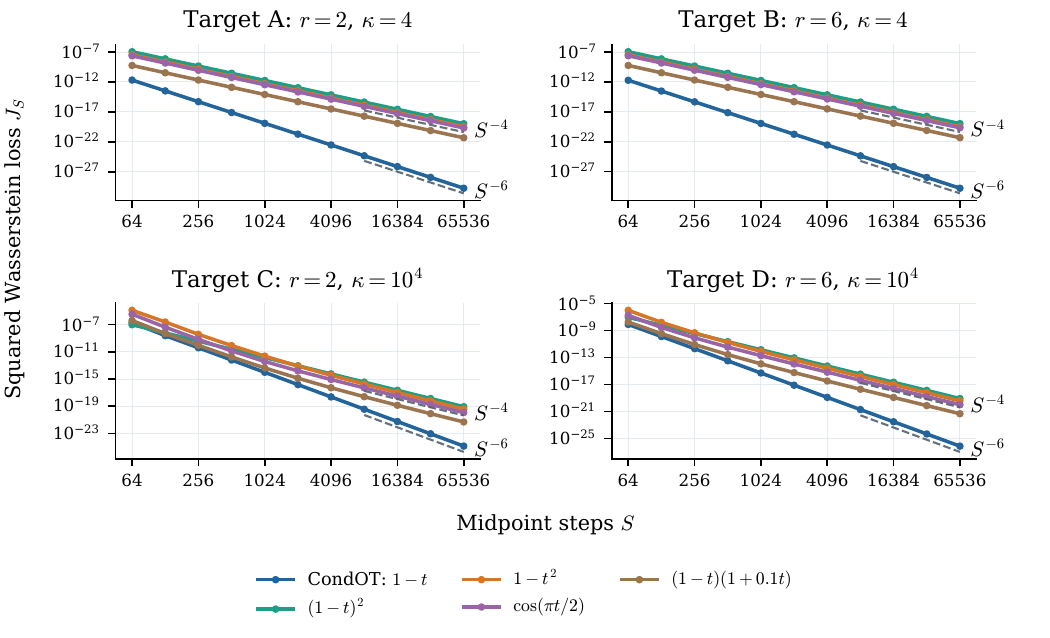}
\caption{Fixed-schedule convergence for all four targets. Each panel repeats
the comparison in Figure~\ref{fig:accuracy}; the dashed lines indicate
$S^{-4}$ and $S^{-6}$.}
\label{fig:all-fixed}
\end{figure}

\newpage

Figure~\ref{fig:all-correction} repeats the comparison between CondOT and
the first-order correction in Figure~\ref{fig:leading-correction} for every target.

\begin{figure}[H]
\centering
\includegraphics[width=0.92\linewidth]{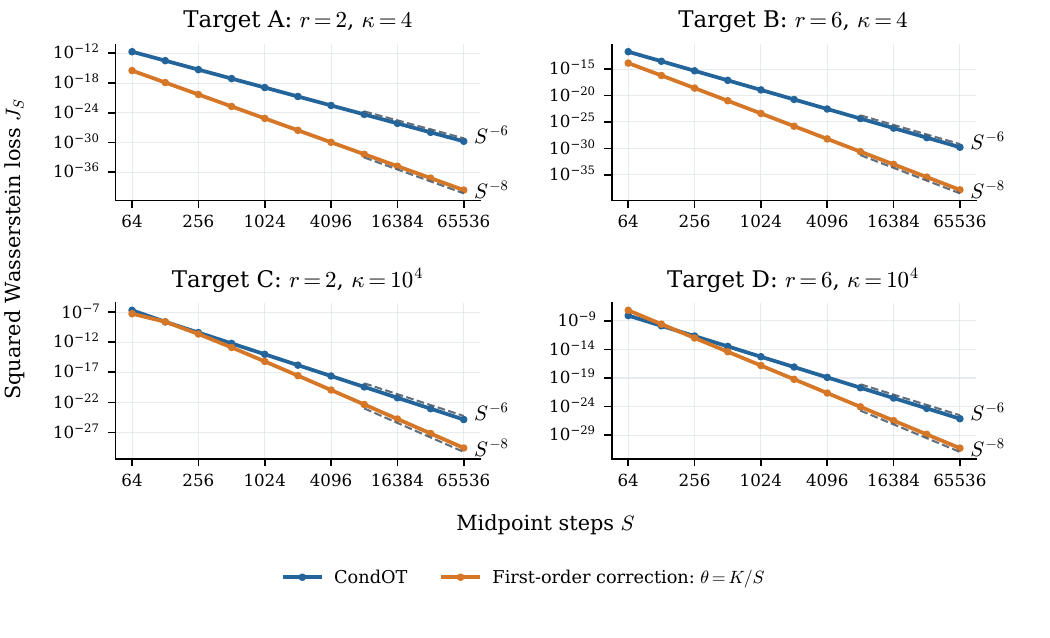}
\caption{CondOT and the first-order correction for all four targets.
The dashed lines indicate $S^{-6}$ and $S^{-8}$. The vertical limits vary
between panels.}
\label{fig:all-correction}
\end{figure}

\subsection{Calibrated schedule shapes and coefficients}
\label{app:four-crossing-numerics}

Table~\ref{tab:coefficients} gives rounded coefficients for the schedules
shown in Figure~\ref{fig:four-crossings}, all at $S=256$. 

\begin{table}[H]
\centering
\caption{Calibrated coefficients at $S=256$, rounded to eight decimal places.}
\label{tab:coefficients}
\setlength{\tabcolsep}{3pt}
\small
\begin{tabular}{crrrrrr}
\toprule
Target & $\theta_1$ & $\theta_2$ & $\theta_3$ & $\theta_4$ & $\theta_5$ & $\theta_6$\\
\midrule
A & $-0.00160120$ & $0.00153963$ & -- & -- & -- & --\\
B & $0.00578858$ & $0.02958382$ & $-0.38062679$ & $1.13618161$ & $-1.39374372$ & $0.60501822$\\
C & $-0.04198104$ & $-0.14789871$ & -- & -- & -- & --\\
D & $-0.45092399$ & $-1.70655900$ & $26.80927083$ & $-73.25730884$ & $76.48650974$ & $-27.84974968$\\
\bottomrule
\end{tabular}
\end{table}

The schedule difference is
\[
 \beta_{\boldsymbol\theta}(t)-(1-t)
 =t(1-t)(\theta_1+\theta_2t+\cdots+\theta_rt^{r-1}).
\]
We locate the sign-changing roots of the polynomial above numerically
and refine them using 90-digit arithmetic. Targets A and C have no
interior crossings; B has two, near $0.3948$ and $0.9806$; and D has
one, near $0.2868$.

We compute the maximum absolute deviations in
Section~\ref{m:crossing-experiment} by evaluating the schedule difference
at its endpoints and numerically located stationary points.
The numerical checks also show that all four calibrated schedules
decrease monotonically and have positive midpoint factors.
The calibration theorem does not provide a budget threshold ensuring monotonicity or determine the number of interior crossings.

\subsection{Calibration in ordinary double precision}
\label{app:precision-check}

We repeat the procedure in Appendix~\ref{app:calibration-numerics}
using ordinary double precision (IEEE binary64) throughout, with a
scale-error tolerance of $10^{-14}$ and at most 1000 updates.
To reduce roundoff near convergence, we sum the logarithms of the
midpoint factors and use accurate evaluations of $\log(1+x)$ and
$\exp(x)-1$ for small $x$.

Table~\ref{tab:precision} reports errors reevaluated at 90 digits
using the resulting coefficients, without refining them.
Sixteen of the twenty runs converge, and their reevaluated scale
errors remain below $10^{-14}$, with squared Wasserstein losses
below $3.0\times10^{-28}$.
The same four small-budget cases fail as in
Appendix~\ref{app:calibration-numerics}.

\begin{table}[H]
\centering
\caption{Proof iteration computed in binary64 and reevaluated at 90 digits
without coefficient refinement. The scale error is
$\max_k|D_{1,k}^{\mathrm M}-1/\sqrt{\lambda_k}|$; $J_S$ includes eigenvalue
multiplicities. Capped runs report the final iterate's errors after 1000
updates. Dashes indicate a nonpositive midpoint factor.}
\label{tab:precision}
\begin{tabular}{crrrrl}
\toprule
Target & $S$ & Updates & Maximum scale error & $J_S$ & Solver status\\
\midrule
A & $64$ & 4 & $1.33\times10^{-15}$ & $5.83\times10^{-30}$ & Converged\\
A & $128$ & 3 & $1.16\times10^{-15}$ & $6.65\times10^{-30}$ & Converged\\
A & $256$ & 2 & $9.67\times10^{-16}$ & $2.82\times10^{-30}$ & Converged\\
A & $512$ & 2 & $7.36\times10^{-17}$ & $2.47\times10^{-32}$ & Converged\\
A & $1024$ & 1 & $3.12\times10^{-16}$ & $5.08\times10^{-31}$ & Converged\\
\addlinespace[3pt]
B & $64$ & 7 & $7.76\times10^{-16}$ & $1.07\times10^{-30}$ & Converged\\
B & $128$ & 5 & $8.56\times10^{-16}$ & $2.59\times10^{-30}$ & Converged\\
B & $256$ & 3 & $1.90\times10^{-15}$ & $7.49\times10^{-30}$ & Converged\\
B & $512$ & 2 & $2.95\times10^{-16}$ & $3.12\times10^{-31}$ & Converged\\
B & $1024$ & 2 & $9.01\times10^{-17}$ & $1.33\times10^{-32}$ & Converged\\
\addlinespace[3pt]
C & $64$ & 1000 & $1.07\times10^{-5}$ & $6.83\times10^{-10}$ & Cap reached\\
C & $128$ & 1000 & $2.12\times10^{-5}$ & $1.35\times10^{-9}$ & Cap reached\\
C & $256$ & 56 & $9.96\times10^{-15}$ & $2.98\times10^{-28}$ & Converged\\
C & $512$ & 20 & $5.69\times10^{-15}$ & $9.70\times10^{-29}$ & Converged\\
C & $1024$ & 10 & $7.22\times10^{-15}$ & $1.57\times10^{-28}$ & Converged\\
\addlinespace[3pt]
D & $64$ & 4 & -- & -- & Invalid factor\\
D & $128$ & 8 & -- & -- & Invalid factor\\
D & $256$ & 49 & $6.49\times10^{-15}$ & $1.04\times10^{-28}$ & Converged\\
D & $512$ & 14 & $8.37\times10^{-15}$ & $1.07\times10^{-28}$ & Converged\\
D & $1024$ & 8 & $2.27\times10^{-15}$ & $7.69\times10^{-30}$ & Converged\\
\bottomrule
\end{tabular}
\end{table}

\end{document}